\documentclass{article}

\usepackage{microtype}
\usepackage{graphicx}
\usepackage{booktabs} 
\usepackage{subcaption}
\usepackage{graphicx}
\usepackage{algorithm}
\usepackage{algorithmic}

\usepackage{amssymb}
\usepackage{amsmath}
\usepackage{amsthm}
\usepackage[dvipsnames]{xcolor}
\usepackage[short,c2]{optidef}
\usepackage{bbm}
\usepackage{subcaption}
\usepackage{wrapfig}
\usepackage{hyperref}

\newcommand{\R}{\mathbb{R}}  
\newcommand{\E}{\mathbb{E}}  
\newcommand{\calC}{\mathcal{C}}
\newcommand{\calU}{\mathcal{U}}
\newcommand{\calQ}{\mathcal{Q}}
\newcommand{\SW}{\mathit{SW}}
\newcommand{\cSW}{\mathit{cSW}}
\newcommand{\MR}{\mathit{MR}}
\newcommand{\MC}{\mathit{MC}}
\newcommand{\mC}{\mathcal{C}}
\newcommand{\mU}{\mathcal{U}}
\newcommand\abs[1]{\left| #1  \right|}

\newcommand{\Rgrt}{\mathit{Rgrt}}

\newcommand{\bfr}{\mathbf{r}}
\newcommand{\ignore}[1]{}
\newcommand{\commentout}[1]{} 
\mathchardef\mhyphen="2D

\newtheorem{theorem}{Theorem}
\newtheorem{proposition}{Proposition}

\newtheorem{observation}{Observation}
\newtheorem{lemma}{Lemma}
\newtheorem{definition}{Definition}

\usepackage[accepted,nohyperref]{icml2020}

\icmltitlerunning{Optimizing Long-term Social Welfare in Recommender Systems}

\begin{document}

\twocolumn[
\icmltitle{Optimizing Long-term Social Welfare in Recommender Systems:\\A Constrained Matching Approach}

\icmlsetsymbol{student}{*}
\icmlsetsymbol{intern}{$\dagger$}

\begin{icmlauthorlist}
\icmlauthor{Martin Mladenov}{goo}
\icmlauthor{Elliot Creager}{to,ve,student}
\icmlauthor{Omer Ben-Porat}{tec,intern}
\icmlauthor{Kevin Swersky}{goo}
\icmlauthor{Richard Zemel}{to,ve}
\icmlauthor{Craig Boutilier}{goo}
\end{icmlauthorlist}

\icmlaffiliation{to}{University of Toronto}
\icmlaffiliation{ve}{Vector Institute}
\icmlaffiliation{goo}{Google Research}
\icmlaffiliation{tec}{Technion}

\icmlcorrespondingauthor{Martin Mladenov}{mmladenov@google.com}

\icmlkeywords{Recommenders, Trustworthy Machine Learning, Combinatorial Optimization, ICML}

\vskip 0.3in
]

\printAffiliationsAndNotice{\icmlEqualContribution}

\begin{abstract}
Most recommender systems (RS) research assumes that a user's utility can be maximized independently of the utility of the other agents (e.g., other users, content providers). In realistic settings, this is often not true---the dynamics of an RS ecosystem couple the long-term utility of all agents. In this work, we explore settings in which content providers cannot remain viable unless they receive a certain level of user engagement.
We formulate the recommendation problem in this setting as one of \emph{equilibrium selection} in the induced dynamical system, and show that it can be solved as an \emph{optimal constrained matching} problem. 
Our model ensures the system reaches an equilibrium with maximal social welfare supported by a sufficiently diverse set of viable providers.
We demonstrate that even in a simple, stylized dynamical RS model, the standard \emph{myopic} approach to recommendation---always matching a user to the best provider---performs poorly.
We develop several scalable techniques to solve the matching problem, and also draw connections to various notions of user regret and fairness, arguing that these outcomes are fairer in a utilitarian sense.
\end{abstract}

\section{Introduction}
\label{sec:intro}
Investigations of various notions of \emph{fairness} in machine learning (ML) have shown that,
without due care, applying ML in many domains can result in biased outcomes that disadvantage
specific individuals or groups \cite{Dwork,barocasBigData}.
Content \emph{recommender systems (RSs)},
which match users
to content (e.g., news, music, video), typically rely on ML to
predict a user's interests 
to recommend ``good'' content
\citep{grouplens:cacm97,jacobson2016music,covington:recsys16}. 
Since these predictions
are learned from past behavior, many issues of ML fairness arise in RS settings \citep{beutel_etal:kdd19}.

One aspect of ``fairness'' that has received little attention emerges when one considers the \emph{dynamics
of the RS ecosystem}. Both users and content providers have particular incentives for engaging with an RS
platform---incentives which interact, via the RS \emph{matching policy},
to couple the long-term utility of agents on both sides of this
content ``marketplace.''  Some work has looked at the impact of RS policies on provider welfare 
\cite{SinghJoachims2018}, on long-term user and provider
metrics (e.g., using RL \citep{chen_etal:2018top,slateQ:ijcai19}
or assessing various other phenomena \cite{Ribeiro_etal:FAT20,celma:longtail2010}). However,
little work has looked
at the interaction of the two on the ecosystem dynamics induced by the RS policy (though there are some exceptions, e.g., \citep{benporat_etal:nips18}, which are discussed below).

In this work, we focus on \emph{provider} behavior using
a stylized model of a content RS ecosystem in which providers 
require a certain degree of user engagement (e.g., views, time spent, satisfaction) to remain \emph{viable}. This required
degree (or ``threshold'') of engagement
reflects their incentives (social, economic, or otherwise) to participate in the RS;
and if this threshold is not met, 
a provider will withdraw from the platform (i.e., their content is
no longer accessible).
If this occurs, user segments for whom
that provider's content is ideal may be disadvantaged, leading such users to derive less utility from
the RS. 
This may often arise, say, for users and providers of \emph{niche} content.

Typical RS policies are \emph{myopic}: given a user request or query, it returns the provider that is (predicted to be) best-aligned with that query. In our
model, myopic policies often drive the dynamical system to a poor equilibrium,
with low \emph{user social welfare} and poor provider diversity. By contrast,
a more \emph{holistic} approach to matching requests to providers
can derive much greater user welfare. 
We formulate policy optimization in our model
as a constrained matching problem, which 
\emph{optimizes for the socially optimal equilibrium of the ecosystem}.
We develop scalable techniques for
computing these policies, and show empirically that they produce much
higher social welfare than myopic policies, \emph{even in this relatively stylized model}.
Such policies also lead to greater provider diversity---even though
our objective only involves \emph{user} utility---by (implicitly) ``subsidizing'' some providers
that would not remain viable under the myopic policy.  
We examine tradeoffs between user regret and social welfare in this model,
showing that user maximum regret tends to be quite low,
especially with \ignore{realistic} utility functions exhibiting diminishing returns. Finally, we draw connections to notions of ML fairness and argue that
the outcomes induced by our matching-based policies are fairer in a utilitarian sense.

\section{Challenges for Myopic Content Matching}
\label{sec:challenges}
We first introduce an abstract, but general, formalization of dynamic content RSs.
Two key elements drive ecosystem utility
and dynamics. First, users derive (possibly non-linear) utility from the \emph{collection}
or \emph{sequence}
of content recommended to them, not just from the individual items.
Second, providers require minimum levels of user engagement to remain \emph{viable}, i.e., incentivized to engage with the RS. 
Within this model, we show how
\emph{myopic} RS policies often poorly serve both users and providers,
and discuss policy types that overcome this.

\subsection{A Formalization of Dynamic Recommendations}
\label{sec:formalization}

We assume an RS (or platform) that matches \emph{users} $\calU$ to \emph{content providers} $\calC$. Users issue 
\emph{queries} for desired content, drawn from
some space $\calQ$. We assume $\calU, \calC, \calQ$ are finite.
Given user $u$'s query $q_u$, the RS returns a provider
$c$ from which $u$ derives \emph{immediate reward} $r(q_u,c)$ reflecting match
quality. We often
assume the existence of some \emph{latent space} $X\subseteq\R^d$ that is used to
represent both content and queries. This is common, say, in
collaborative filtering (CF), where $X$ is an
embedding space constructed by
matrix factorization \cite{salakhutdinov-mnih:nips07} or neural CF
\cite{he_etal:www17,beutel_etal:wsdm18}). 
For
$c, q_u \in X$ we let $r(q_u,c) = q_u^T c$ be immediate reward.

User queries are received asynchronously and immediately matched to a provider,
giving a discrete-event
dynamical system over time periods $1,2,\ldots,t,\ldots$  At each time $t$,
a user $u[t]$, drawn from some distribution $\rho(\calU)$, issues a query $q_u[t]\in X$,
itself drawn from distribution $P_{u[t]}(X)$ reflecting
that user's interests. 
The RS \emph{matches} $q_u[t]$ to some 
provider $c[t]\in \calC$, and $u[t]$ derives reward $r(q_u[t],c[t])$.\footnote{For
ease of exposition, we do not distinguish the different content \emph{items} offered by
provider $c$. Nothing fundamental changes in our approach if we match queries to specific
items if each item is associated with a provider.}
We assume the RS has \emph{complete knowledge} of the 
location in $X$ of $c$ and $q_u$,
and that these embeddings do not change over time.
While unrealistic in practice---most RSs continually update user and content representations, and may engage in active exploration to help assess user interests---
the problems we address under this assumption are further exacerbated by incomplete information. We discuss this further below.

Let
$h_t = ((u[i],q_u[i],c[i]))_{i\leq t}$ be a length $t$ history of past queries and recommendations, $H[t]$ be the
set of all such histories, and $H[\ast] = \cup_{t<\infty} H[t]$. An RS \emph{policy}
${\pi: H[\ast]\times (\calQ,\calU)\rightarrow\Delta(\calC)}$ maps a history and a query to a distribution
$\pi(h_t,q_u[t])$ over providers. 

Generally, an RS aims to maximize user engagement. We assume
that over long horizons, user engagement is optimized by maximizing \emph{user utility},
which in turn drives sustained user satisfaction with the platform. 
We assume that $u$'s utility is a (possibly non-linear) function $f$ of the reward
sequence  $\bfr_u$ obtained over some horizon. 
For instance, this might be the cumulative sum of rewards; recency-weighted cumulative reward; or a function (e.g., sigmoid)
that captures various behavioral phenomena, such as decreasing marginal returns. We discuss several such functions in Sec.~\ref{sec:formulations}.
However, even with this user focus, the RS must also address the provider
incentives---in our case,
viability---since providers offer the quality content needed for a thriving ecosystem.
Each provider $c$ has a \emph{viability threshold} $\nu_c$ over the amount of
(possibly recency-weighted)
user engagement (e.g., visits, time spent) generated for it by the RS. 
Periodically, $c$ compares its overall engagement against $\nu_c$
and abandons the platform---perhaps stochastically--if this threshold is not reached (we detail specific forms of this function below). 
If a provider becomes unviable, it can no longer
be matched by the RS to any user query.

\subsection{Suboptimality of Myopic Policies}
\label{sec:illustrations}

Before developing our methods, we use two simple examples to
illustrate why typical \emph{myopic} recommendation policies may be suboptimal.

\begin{figure}[t!]
  \centering
  \includegraphics[width=.42\textwidth]{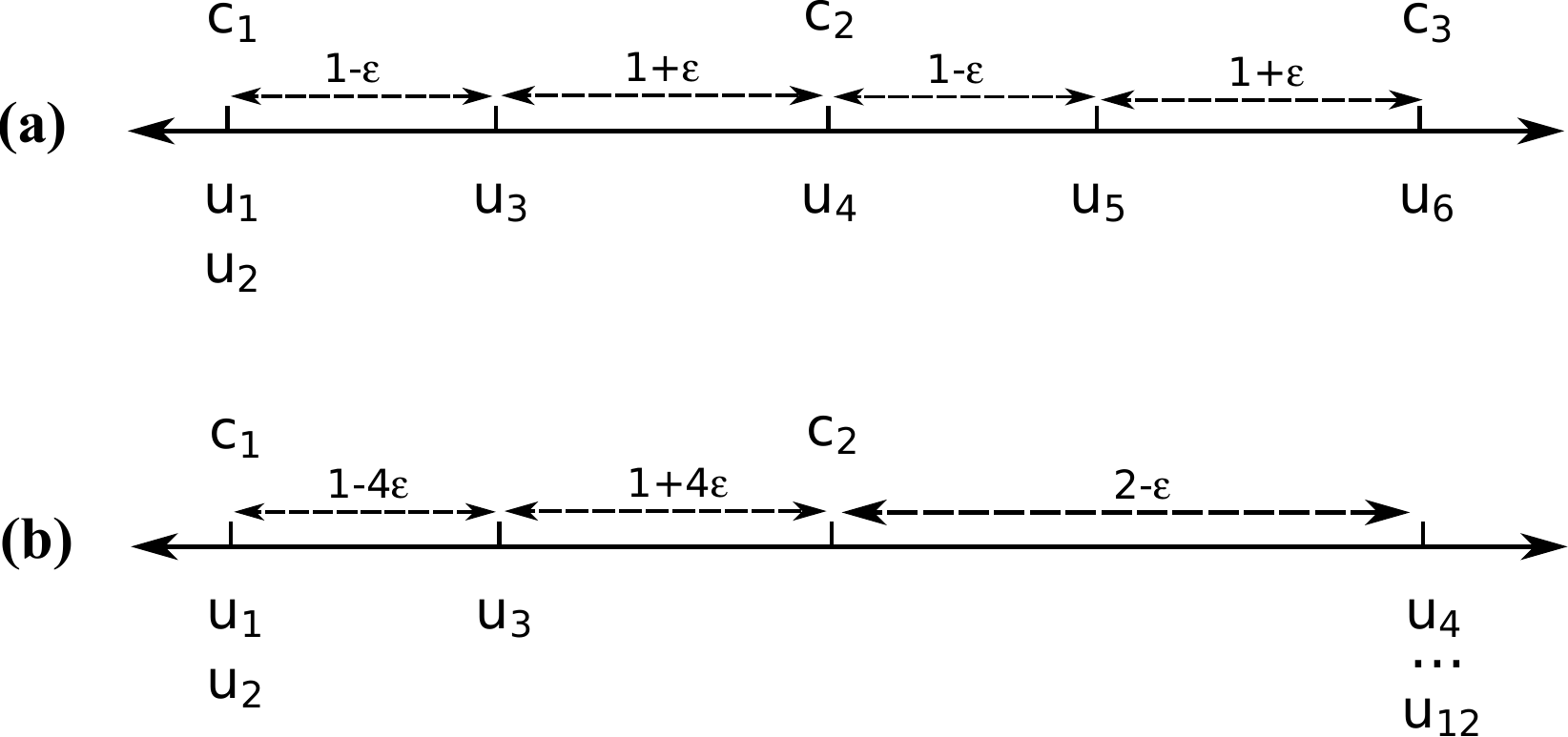}
  \caption{Two 1D examples: (a) User $u_i$'s reward for being matched to provider $c_j$ is 2 less the distance between them (e.g., $u_1$ has reward 2 for $c_1$; $u_3$ has reward $1+\varepsilon$ (resp., $1-\varepsilon$) for $c_1$ (resp., $c_2$). We equate reward and utility, and a user with her query.
  assume that each user issues a
  single query per period, and that each provider requires 2 user impressions in each period to remain
  viable at the next period. (b) Similar to (a) except that $c_1$ requires 2 impressions and $c_2$ requires 10.
  }
  \label{fig:SW_example1}
\end{figure}

Typical RSs behave \emph{myopically}: when query $q_u$ is received, it is matched to the
provider giving maximum user reward $c^\ast_{q_u} = \arg\max_{c\in\calC} r(q_u, c)$. In the example in
Fig.~\ref{fig:SW_example1}(a), each provider has a viability threshold of $2$ and is viable in the initial period.
The myopic policy first matches each of $u_1, u_2, u_3$ to $c_1$; $u_4, u_5$ to $c_2$; and $u_6$ to $c_3$.
This gives total immediate reward
$10+2\varepsilon$ according to the reward function ${r(u, c) = 2 - |u - c|}$ for each recommendation.
However, since $c_3$ receives only one user, it is no longer viable at the next period. Hence, at all subsequent periods, the RS must
match $u_6$ to $c_2$ (with reward 0),
attaining a long-run per-period average reward of $8+2\varepsilon$.

A \emph{non-myopic policy} can obtain a long-run average reward of $10-2\varepsilon$ by matching $u_3$ to $c_2$
and $u_5$ to $c_3$ at each period. Under this policy,
$c_3$ remains viable, allowing $u_6$ to receive
reward 2 (rather than 0) in perpetuity. This comes at a small price to $u_3$ and $u_5$, each of whom receive
$2\varepsilon$ less per period. This matching \emph{subsidizes} $c_3$ by matching its content to
$u_5$ (who would slightly prefer provider $c_2$). This subsidy leaves
$c_2$ vulnerable, so it too is subsidized by the match with $u_3$. Indeed,
this
matching is \emph{optimal} for any horizon of at least two periods---its average-per-period 
\emph{user social welfare} (or total reward) is maximized. The maximum loss of utility experienced by
any user at any period w.r.t.\ the myopic policy is quite small, only $2\varepsilon$ (by both $u_3, u_5$)---this is the \emph{maximum (user) regret} of the policy. Finally, this policy keeps all providers viable in perpetuity; the set of viable providers $V=\calC$ is an \emph{equilibrium} of the
dynamical system induced by the policy. By contrast, the myopic policy reaches an equilibrium $V'=\{c_1,c_2\}$ that has fewer viable providers.

Consider now a policy that matches $u_3$ to $c_3$ at each period, but otherwise behaves myopically. This induces 
the same equilibrium $V=\calC$ as the optimal policy by subsidizing $c_3$ with $u_3$.
However, this policy---though improving $u_5$'s utility by $2\varepsilon$ (and her regret to $0$)---gives a reward of $0$ to $u_3$ (whose regret is $1+\varepsilon$). This policy not
only has higher max regret, it also has significantly lower welfare of $9+\varepsilon$.

While not the case in this example, the policy that optimizes social welfare need not minimize max regret. 
Fig.~\ref{fig:SW_example1}(b) considers a case where the viability threshold differs from each of the two providers.
The myopic policy initially matches $\{u_1, \ldots, u_3\}$ to $c_1$ and $\{u_4, \ldots, u_{12}\}$ to $c_2$, after which $c_2$ is no longer viable (and $\{u_4, \ldots, u_{12}\}$ receive no further reward). 
Thus per-period reward is $5+4\varepsilon$ (and max regret is $\varepsilon$.) 
The welfare-optimal policy subsidizes $c_2$ by matching $u_3$, increasing welfare marginally by $\varepsilon$ to $5+5\varepsilon$, but also increasing max regret (see $u_3$) to $8\varepsilon$. 
This illustrates the trade-off between social welfare maximization and max-regret minimization.

These examples show that maximizing user
social welfare often requires that the RS take action
to \emph{ensure the long-run viability of providers}. 
The example from Fig~\ref{fig:SW_example1}(a) shows
that such considerations need not be explicit, but simply emerge as 
a \emph{by-product of maximizing user welfare alone}. This also promotes
diversity among viable providers that can in some sense be
interpreted as being ``more fair'' to the user population. In particular, it creates a smaller gap between the (long-term) utility values attained by different users across the spectrum of possible topic interests. However, as with provider diversity, this type of fairness is not part of the \emph{explicit} objective that drives the RS policy---rather it is implicit, with fairness emerging as a consequence of trying to maximize \emph{overall} user welfare. We discuss connections to work on ML fairness further below.
For a richer illustration of this, see Fig.~\ref{fig:figure2}.

\newcommand{\synsubfigwidth}{0.3\textwidth}
\begin{figure*}[th!]
  \centering
  \begin{subfigure}[t]{\synsubfigwidth}
    \noindent\resizebox{\textwidth}{!}{
      \includegraphics{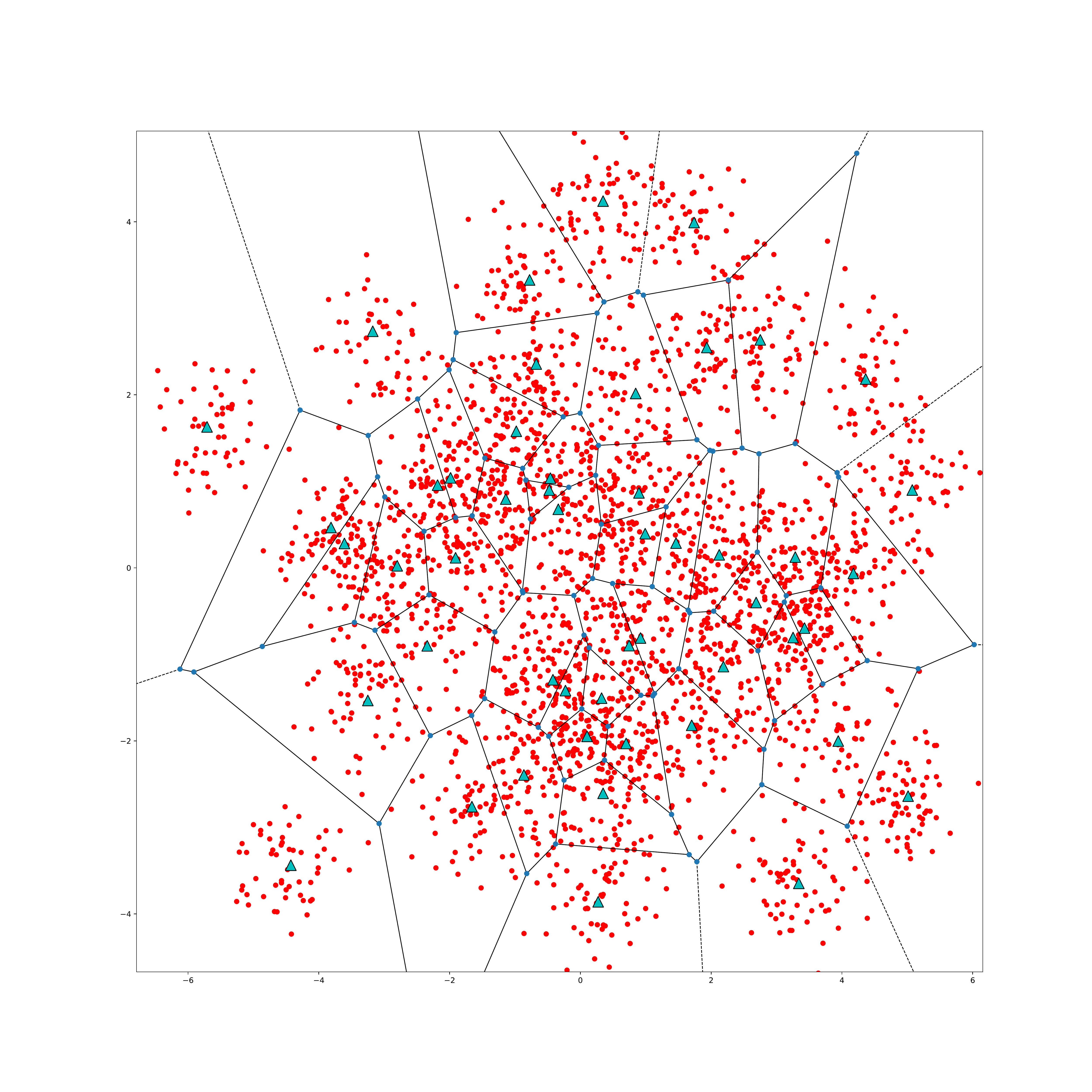}
    }
  \end{subfigure}%
  \hfill
  \begin{subfigure}[t]{\synsubfigwidth}
    \noindent\resizebox{\textwidth}{!}{
      \includegraphics{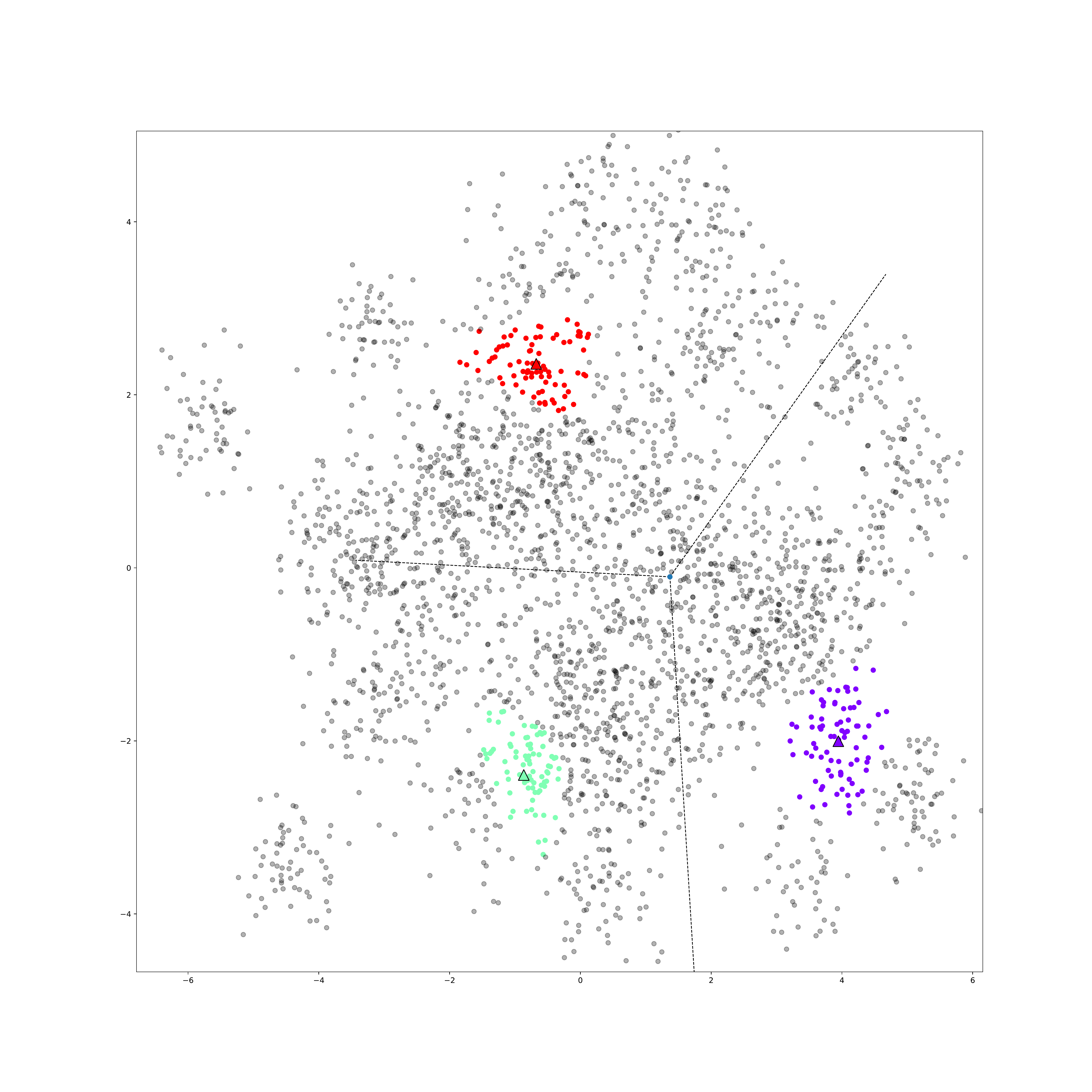}
    }
  \end{subfigure}
  \hfill
  \begin{subfigure}[t]{\synsubfigwidth}
    \noindent\resizebox{\textwidth}{!}{
      \includegraphics{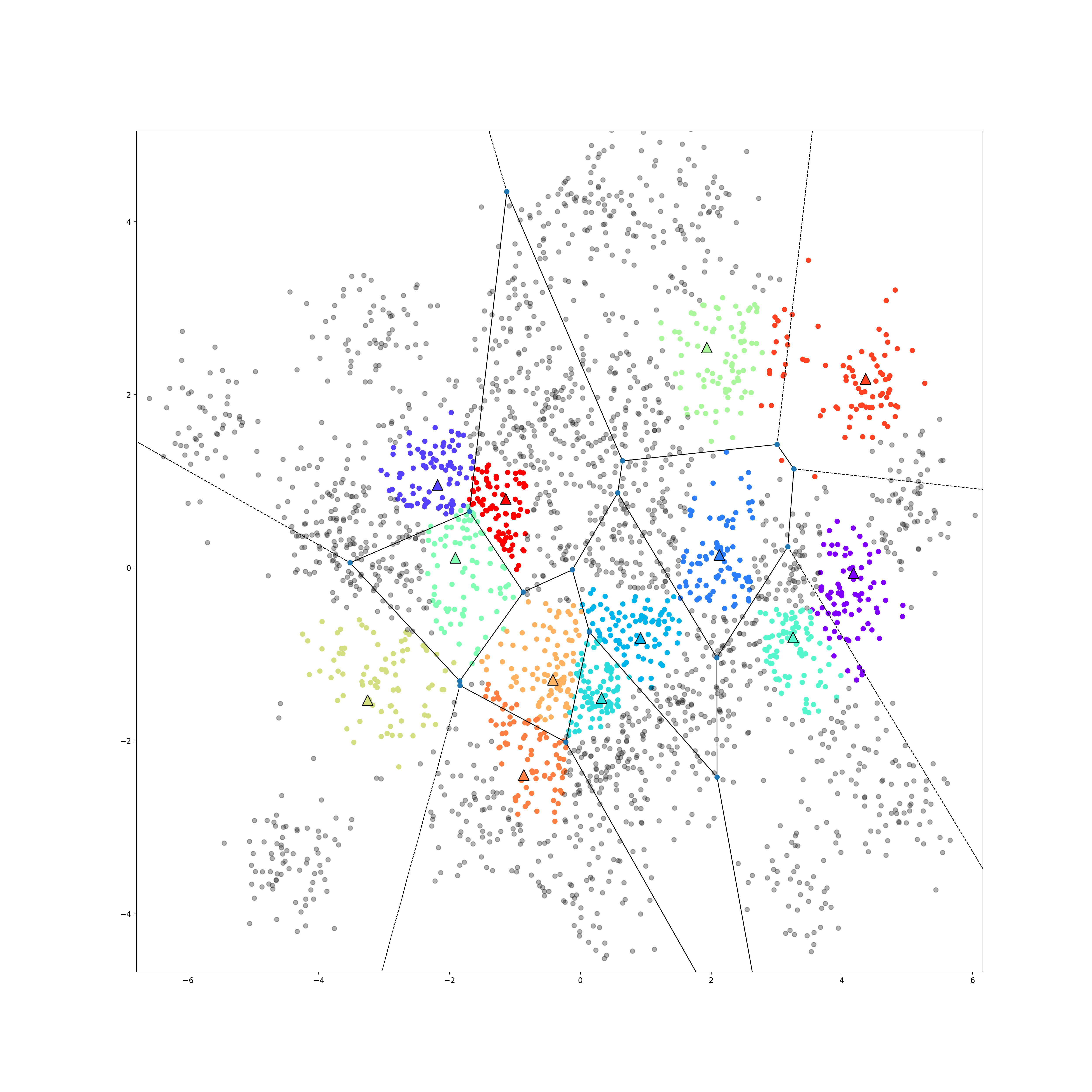}
    }
  \end{subfigure}%
  \caption{
    Recommendation using dynamics-informed constraints keeps more content providers viable in the long run, improving social welfare amongst users.
    The first panel shows providers (blue triangles) and users (red dots) embedded in a 2D topic space.
    The second  and third  panel respectively show equilibrium under Myopic and proposed linear program (LP) recommendation;
    here users whose most-preferred provider has dropped out of the ecosystem are shaded in grey, while the remaining users are colored according to their most-preferred provider.
  }
  \label{fig:figure2}
\end{figure*}

\subsection{Matching Optimization for Recommendation}
\label{sec:matching}

We now formalize our objectives and optimization approach.
For ease of exposition, we assume an \emph{epoch-based} decision problem:
time steps are grouped into
epochs of fixed length $T$, with user utility and provider viability both determined at the end of each epoch. Let $Q$ denote the induced distribution
over queries during any epoch (since user behavior is stationary, so is $Q$).
Other forms of
user/provider evaluation do not impact the qualitative nature of our
results---some require different forms of analysis and optimization, while others carry through easily. For example, if providers use recency-weighted engagement, no substantial changes are needed; but if their evaluation occurs on a continual (not epoch-based) basis, more intricate equilibrium analysis is required and optimization becomes more online in nature. 

A policy $\pi$ induces a stochastic dynamical system over a state space, where the state encodes user utility and provider viability at the end of each epoch. Let random variable (RV) $e^\pi_t(c)$ be provider $c$'s engagement at time $t$ under $\pi$, and $E^\pi_k(c)$ its cumulative engagement during epoch $k\geq 1$. If
$E^\pi_k(c) \geq \nu_c$, $c$ remains viable at epoch $k+1$, otherwise it abandons the platform. Let $V^\pi_k$ be the set of providers that are viable at the end of
epoch $k$, We assume $V^\pi_0 = \calC$.

Let (RV) $\bfr^\pi_k(u)$ be user $u$'s reward sequence
in epoch $k$ under $\pi$, and 
$U^\pi_k(u) = f(\bfr^\pi_k(u))$ be $u$'s utility. \emph{Social welfare} generated by
$\pi$ at epoch $k$ is ${\SW^\pi_k = \sum_{u\in\calU} U^\pi_k(u)}$.
(Long-run) \emph{average social welfare} is
${\SW^\pi_\infty = \lim_{k\rightarrow\infty} [\sum_k \SW^\pi_k] / k}$.
The average utility $U^\pi_\infty(u)$ of $u$ under $\pi$ is defined
analogously. If $\pi^\ast_u$ is the policy that maximizes
$u$'s average utility, then $u$'s \emph{regret} under $\pi$ is
$\Rgrt^\pi(u) = U^{\pi^\ast_u}_\infty(u) - U^\pi_\infty(u)$.
The \emph{maximum regret} of $\pi$ is
$\MR(\pi) = \max_{u\in\calU} \Rgrt^\pi(u)$.
Let $\MC(u)\subseteq \calC$ be those providers matched by
the myopic policy to queries $\calQ_u$ with positive support
in $P_u$ when all providers are viable.
Under the mild assumption that, for any $u$, there is a policy
that keeps $\MC(u)$ viable,
$U^{\pi^\ast_u}_\infty(u) = \E [f(\bfr^\ast(u))]$ is a \emph{constant} where
$\bfr^\ast(u)$ is the realized reward sequence when any 
$q_u$ generates reward $\arg\max_c r(q_u,c)$.

Our interest in long-run performance leads to a focus on policy
behavior in equilibrium w.r.t.\ provider viability.
We say $V\subseteq\calC$ is an \emph{equilibrium} of $\pi$ if, for
some $k\geq 0$, $V^\pi_{k+j} = V$ for all $j\geq 0$ (i.e.,
the providers $V^\pi_k$ that are viable after epoch $k$ remain
viable in perpetuity). 
Since most large-scale RSs have massive numbers of users, we assume
that the number of queries received at each point in $X$ (or within a suitable set of small subregions of $X$) during each
epoch is exactly its expectation. This can be justified by appeal to the law of
large numbers for epochs of sufficient duration.\footnote{We discuss relaxations
of this assumption below.}

Under these assumptions, the recommendation policy that maximizes
average social welfare has an especially simple form. Specifically, there is
an optimal \emph{stationary policy} (w.r.t.\ epochs) which can
be formulated as an optimal matching problem under viability constraints.
Consider the following \emph{single-epoch decision problem}:
\begin{align}
\max_{\pi} & \sum_{u \in {\cal U}}
\E [ f(\bfr^\pi(u)) | \pi ] \label{eq:match1Obj}\\
\text{s.t.\ } & \pi_{q,c} > 0 \text{ only if } 
  \sum_{u, q_u} \overline{Q}(q_u)\pi_{{q_u},c} \geq \nu_c, \,\, \forall \ q, c
  \label{eq:match1Constr1}
\end{align}
Here $\pi$ is a vector of matching variables $\pi_{{q_u},c}$, denoting
the proportion of queries $q_u$ to match to $c$.\footnote{This can 
also be interpreted as a stochastic policy, which is the natural
practical implementation.
When user utility is time-dependent, we sometimes allow the policy to be
non-stationary \emph{within} the epoch, writing $\pi_{{q_u},c,t}$ for $t\leq T$.}
$\overline{Q}(q_u)$ is the expected number of queries of the type $q_u$,
but note that the method holds for any non-negative function of $q_u$.
The expectation of $u$'s utility $f$ is
taken w.r.t.\ user activation at each $t$ in the epoch, the user query
distribution and $\pi$; i.e., the $t$-th component of $\bfr^\pi(u)$ is
distributed as
\begin{equation}
\label{eq:reward_dist}
r^\pi_t \sim \rho(u)\sum_{q_u \in \calQ} P_u(q_u) \sum_c \pi_{{q_u},c} r(q_u,c).
\end{equation}
Objective~(\ref{eq:match1Obj}) optimizes social welfare over the epoch,
while constraint~(\ref{eq:match1Constr1}) ensures that any \emph{matched} provider remains viable at the end of the epoch.

We show that applying this single-epoch policy $\pi$ across all
epochs is guaranteed to optimize average user social welfare.
Since the expectation $\overline{Q}(q_u)$ is realized exactly, $\pi$ induces
an equilibrium during the first epoch. Moreover, the optimization ensures
it has maximum welfare, i.e., is
the optimal stationary policy. Finally, while a non-stationary $\pi'$
may initially improve welfare relative to $\pi$, its equilibrium
welfare cannot exceed that of $\pi$; so any welfare improvement is transient and cannot increase (long-run) average welfare.
The stationary $\pi$ given by~(\ref{eq:match1Obj})
\emph{anticipates the equilibrium it induces} and only matches to providers that
are viable in that equilibrium. This obviates the need to consider more
complex policies based on the underlying Markov decision process
(we discuss richer policy classes below).

\section{Solving the Matching Optimization}
\label{sec:formulations} 

We now develop several practical methods for 
solving the optimization~(\ref{eq:match1Obj}) to generate optimal
policies.
We consider formulations that accommodate
various forms of user utility: simple linear, cumulative
reward; a discounted model that reflects decreasing marginal
returns; and non-linear models (of which sigmoidal utility is a motivating example).
We then show how social welfare and regret can be traded off, and
 briefly describe how the models can be made more robust to uncertainty in
the user query stream.

\subsection{Additive Utility: A Linear Programming Model}
\label{sec:simpleLP}

We first develop a linear programming (LP) model that applies when
user utility is suitably additive, capturing both cumulative
reward and a \emph{discounted engagement} model
that reflects a natural form of decreasing marginal return.

Let $\alpha \in \mathbb{R}^T_+$ be a vector of non-negative weights. We
assume linear user utility
$f_\alpha: \bfr \mapsto \sum_{t=1}^T \alpha_t \bfr_t$,
where user utility in an epoch is the $\alpha$-weighted sum of immediate rewards
(cumulative reward is a special case with all ${\alpha_t = 1}$). Moreover, we consider a class of non-stationary policies that take time itself as their only history feature, ${\pi_{h_t, q_u, c} = \pi_{t, q_u, c}}$ (if we allow $\pi$ to depend on arbitrary statistics, the problem is a full POMDP).

The expected welfare of $\pi$ over the epoch is 
\[\sum_u \E [ f_\alpha(\bfr^\pi(u)) | \pi ] = \sum_{t=1}^T \sum_{u\in \calU}\sum_{q_u \in \calQ}\sum_{c\in\calC}\alpha_t \pi_{t, q_u, c}\bar{r}(q_{u}, c),  \]
where $\bar{r}(q_u, c) = \rho(u)P_u(q_u)r(q_u, c)$.
Since (per-epoch) social welfare $\SW^\pi$ is a linear function of $\pi$, Eq.~\eqref{eq:match1Obj} can be reformulated as a mixed-integer linear program (MILP):
\begin{maxi}[3]
{\pi, y}{\sum_{t=1}^T \sum_{u\in \calU}\sum_{q_u \in \calQ}\sum_{c\in\calC}\    \alpha_t\pi_{t,q_u, c}\bar{r}(q_{u}, c)}
{\label{eq:ilp_linear_util}}{}
\addConstraint{\sum_c\pi_{q_u, c, t}}{=1\quad}{\forall t \in [1:T], q_u\in \calQ}
\addConstraint{\pi_{q_u, c, t}}{\leq y_c\quad}{\forall t \in [1:T], q_u\in Q, c\in \calC}
\addConstraint{\sum_{u, q_u, t} \overline{Q}(q_u)\pi_{{q_u},c,t}}{\geq \nu_c y_c\quad}{\forall c\in \calC}
\end{maxi}
where matching variables $\pi_{t, q_u, c}$ 
represent the 
stochastic policy, and provider-viability
variables $y_c$ are in $\{0, 1\}$ (for cumulative reward, dependence of $\pi$
on $t$ can be removed).
Problem~\eqref{eq:ilp_linear_util} is akin to \emph{max-utility constrained facility location}, where user-query-time tuples act as customers and providers act as facilities. Related problems have been investigated in various
forms
\cite{chan17facility,li19Facility,nemhauser77Facility}. 
All variants (including ours) have basic facility location as a special case (where the constraints are trivial (i.e., $\nu_c = 0$) and are thus NP-hard. Even though their formulations are similar, each has different approximation properties. The combination of objective and constraints  we consider has not,
we believe, been studied in the literature.   

Our problem can be approximated in polynomial time up to a constant factor:
\begin{theorem}
Problem~\eqref{eq:ilp_linear_util} can be approximated up to factor $\frac{1}{e}$ in polynomial time.  
\end{theorem}
The core of the proof 
(see Appendix~\ref{sec:proof_additive}) is to consider the problem of
computing \emph{maximum
constrained welfare}, $\cSW(C)$, given a \emph{fixed} set $C\subseteq\calC$ of
viable providers. $\cSW(C)$ is the maximum of Eq.~\eqref{eq:ilp_linear_util} when the provider variables are ``set'' as $y_c = 1$ iff $c\in C$; i.e., we
find the best stochastic matching given that all
and only providers in $C$ remain viable. 
With the integer variables removed, the problem becomes a polynomially
sized LP. We show that $\cSW$ is submodular in the provider set $C$,  
i.e., $\cSW(C\cup\{c^\prime, c\}) - \cSW(C\cup\{c^\prime\}) \leq \cSW(C\cup\{c\} ) -\cSW(C)$ for any $c, c^\prime \in \calC, c\subset \calC$. This means that social welfare can be approximately optimized by greedily adding providers until no further viable providers can be added.

A more efficient alternative to greedy provider selection is to directly round the results of the LP relaxation of Eq.~\eqref{eq:ilp_linear_util}. This approach has been found to perform well on similar problems (see e.g. \citet{jones15facility}). In preliminary experiments, w that the LP-rounding heuristic performs indistinguishably from the greedy method in terms of social welfare. Given its superior computational performance, we only evaluate the LP rounding approach in our experiments (denoted LP-RS) in Sec.~\ref{sec:experiments}. Note the techniques for scaling up submodular maximization and linear programming are relatively well understood 
(e.g. ~\citet{JMLR:v17:mirzasoleiman16a, boyd_admm}), hence we will not discuss large-scale implementations within this paper.   

The weighted linear utility model provides us with a  mechanism for modeling
\emph{decreasing marginal returns} in user utility, specifically, by
discounting user rewards by setting $\alpha_t = \gamma^{t-1}$ for some $\gamma \in (0, 1]$. Such discounting is one simple way to model the realistic assumption of
decreasing marginal utility with increased content consumption 
(e.g., a good match for a second user query has less impact on her utility than
it does for the first query). 
This model makes it easier to
maintain provider viability and improve social welfare with lower
individual regret (see Sec.~\ref{sec:experiments}).

\subsection{Non-linear Utility: A Column Generation Model}
\label{sec:columnGeneration}

While additive utility provides us with a useful model
class that can be solved approximately in polynomial time,
it cannot express important structures such as sigmoidal utility
\cite{kahneman_prospect:1979}. Optimal matching with such non-linearities
can be challenging.

Suppose now that $f$ is nonlinear, e.g., $f(\bfr^\pi(u)) = \sigma(\bfr^\pi(u) + \beta)$.
The facility location literature provides little guidance for such
problems. While approximation results are known for 
concave problems~\cite{vahab03facility}, it is unclear if
these apply with constraints, a setting that, we believe, has not been investigated.

Our approach linearizes the problem to form a MILP that can subsequently be relaxed to an LP
as follows. Let $C\in \calC^k$ be a $k$-tuple of providers. A pair $(q_u, C)\in \calQ \times \calC^k$ represents a possible answer to user $u$'s $k$ queries identical to $q_u$ by the provider tuple $C$. We call such a tuple a 
\emph{star} $q_uC$. For each star, we use a variable $\pi_{q_uC}$ to represent the policy's match to $q_u$. The linearized objective is then:
\begin{equation}
    \operatorname{maximize}_{\pi, y}\sum_{u\in \calU}\sum_{q_u \in \cal Q}\sum_{C \in \calC^k}\pi_{q_uC}\bar{\sigma}(q_{u}, C)\ ,
\label{eq:ilp_nonlinear_util}
\end{equation}
where $\bar{\sigma}(q_{u}, C)=\rho(u)P_u(q_u)\sigma(q_{u}, C)$. The linear constraints from Eq.~\eqref{eq:ilp_linear_util} are adapted to
these stars. See
Appendix~\ref{sec:non_linear_opt} for the complete
optimization formulation and our use of
column generation to solve this much larger, more complex problem.

\subsection{Incorporating Regret Trade-offs}
\label{sec:regretopt}

As discussed in Sec.~\ref{sec:illustrations}, pure social welfare maximization can
sometimes induce large maximum regret (i.e., high regret for some users). 
Fortunately, max regret can be traded off 
against social welfare directly in the optimization.
Let $\mu_u$ be a constant denoting $u$'s maximum utility  $U^{\pi^\ast_u}_\infty(u)$
(see Sec.~\ref{sec:matching}). Since the term 
${U^\pi(u) = \sum_{t=1}^T \sum_{q_u \in \calQ}\sum_{c\in\calC}\pi_{q_u, c, t}\bar{r}(q_{u}, c, t)}$
denotes $u$'s expected utility
in LP~\eqref{eq:ilp_linear_util}, we can express $u$'s realized regret as a
variable ${\Rgrt_u \mu_u - U^\pi(u)}$ for all $u\in\calU$. Letting variable $\MR$ represent the max
regret induced by the policy, we can add the term $-\lambda\MR$ to the
objective in MILPs~\eqref{eq:ilp_linear_util} and \eqref{eq:ilp_nonlinear_util} where constant $\lambda$ controls the desired welfare-regret
trade-off. Constraining ${\MR \geq \Rgrt_u, \forall u\in\calU}$ ensures $\MR$ takes
on the actual max regret induced by $\pi$.

\subsection{Extensions of the Model}
\label{sec:robust}

Our matching optimization relies on some restrictive, unrealistic assumptions about real-world RSs.
However, our stylized model directly informs  approaches to
more realistic models (and can sometimes be adapted directly).

\textbf{Robustness and RL:} Strong assumptions about query distribution variance motivated
our equilibrium arguments. Even if these assumptions do not hold, our LP formulation
can be extended to generate
``high-probability'' optimal equilibria. For example, lower-confidence bounds on the traffic expected for each provider under the induced policy can be constructed, and
viability thresholds inflated to allow for a margin of safety. This can be 
encoded formally or heuristically within our LP model. A full RL approach offers
more flexibility by dynamically adjusting the queries matched to a provider as
a function of the state of all providers (e.g., exploiting
query variance to opportunistically make additional providers viable,
or ``salvage'' important providers given, say, an unanticipated drop in traffic).
The combinatorial nature of
the state space (state of engagement of each provider) complicates any RL model. Adaptive online methods (e.g., as used in ad auctions \cite{mehta:onlinematching2013}) can exploit optimization techniques
like ours to dynamically adjust the
matching without the full complexity of RL.

\textbf{Incomplete Information:} Our model assumes complete knowledge of both user utility and
``interests'' (via
the query distribution, the reward function and the utility function), and of a provider's position
in ``topic space'' and its utility (via its viability).
In practice, these quantities are learned from data and constantly updated,
are generally never known with full precision, and are often very uncertain. Indeed, one
reason ``unfair'' recommendations arise is when the RS does not undertake sufficient exploration
to discover diversity in user interests (especially for niche users). Likewise an RS is
usually somewhat (though perhaps less) uncertain of a provider's content distribution.
Incorporating this
uncertainty into our model can (a) generate more robust matchings, and (b) drive
exploration strategies that uncover \emph{relevant} user interests and utility given ecosystem
viability constraints.
These are important directions for future work.

\textbf{Richer Dynamics:} 
More complex dynamics exist in real RS ecosystems than are captured by our
simple model, including: arrival/departure of providers/users; evolving user interests
and provider topics/quality; richer provider responses
(e.g., not just abandonment, but quality reduction, content throttling, topic shifts); and
strategic behavior by providers. Our model serves only as a starting point for richer
explorations of these phenomena.

\section{Experiments}
\label{sec:experiments}
We evaluate our LP-rounding method (Sec.~\ref{sec:simpleLP}), dubbed LP-RS, for additive utility models (we consider
both cumulative and discounted reward). 
Preliminary experiments
show that LP-RS runs faster and performs
better than the greedy/submodular approach, so we do not present results for the latter here.  We
provide a detailed evaluation of column generation for
nonlinear models (Sec.~\ref{sec:columnGeneration}) in 
Appendix~\ref{sec:column-gen-expers}.
We compare the policy $\pi_\text{LP}$ based on LP-RS to
a myopic baseline policy\footnote{
Appendix \ref{sec:stochastic_policy_ablation} compares to an
affinity-aware 
stochastic policy.
} $\pi_\text{My}$
w.r.t.\ social welfare, max regret and provider diversity/viability,
assessing both on several domains
using a
RS ecosystem simulator that captures provider viability dynamics.
We outline the simulator, describe our datasets and trained embeddings,
then discuss our findings.
We use the \textsc{RecSim} framework for simulating RS environments \cite{ie2019recsim}, and the main experiments can be found in the \textsc{RecSim} codebase.\footnote{\href{https://github.com/google-research/recsim/blob/master/README.md\#Papers}{\url{https://github.com/google-research/recsim/blob/master/README.md\#Papers}}
}

\subsection{Ecosystem}\label{sec:ecosystem}
The ecosystem simulator captures the provider viability dynamics described in Sec.~\ref{sec:formalization}.
At each epoch,
the RS observes a single query per user, and must serve a ``slate'' of $s$ providers to each user.
A provider $c$ can only occur \emph{once} in each slate
(e.g., the myopic policy rank-orders the $s$ best-fit providers).
The user query at each epoch is sampled as $q_u[t] \sim P_{u[t]}(X)$;
hence, the realized preference changes at each epoch, but the distribution parameters are fixed throughout the simulation.
Providers are static and all viability thresholds have the same value $\nu$.

\subsection{Datasets}\label{sec:datasets}
\paragraph{Synthetic data}
To examine the emergent properties of the ecosystem in a controlled setting, we generate synthetic data in a two-dimensional topic space, like those in Fig.~\ref{fig:figure2}.
The query distribution is a mixture of Gaussians, with
one component for each provider centered at that provider's location
in topic space $X$, and its weight reflecting
the user's affinity for that provider's topic area.
We consider two variants of the mixture model:
    (a) a \emph{uniform} variant where providers/topics are distributed uniformly,
    and all users have the same variance;
    and (b) a \emph{skewed} variant, where (w.l.o.g.) topics near the origin are considered \emph{popular} and receive relatively more users, while topics far from the origin are \emph{niche} with fewer users, but whose users are more \emph{loyal} and exhibit lower variance.
User-provider reward is given by ${f(a, b) = -||a - b||_2}$ (with max value 0).

\paragraph{Movie ratings}\label{sec:movies}
We train an embedding on the Movielens dataset \citep{harper2015movielens} using
non-negative matrix factorization on a sparse matrix of user-movie engagements
(see Appendix~\ref{sec:embeddings-details} for details).
The pair of matrix factors are then used as embedding vectors for users and providers for the simulator (here each provider is a movie).
User-provider rewards are computed as $r(q_u, c) = q_u^Tc$.

\paragraph{Social network}
\label{sec:snap}
We train user and provider embeddings from the SNAP 2010 Twitter follower dataset~\cite{kwak2010twitter}.
This dataset consists of a large list of (followee, follower) pairs, each with a unique account ID.
We designate popular accounts as ``providers,'' and other accounts as ``users,'' then learn a low-dimensional embedding that captures affinity between users and providers via a link prediction task (see Appendix~\ref{sec:embeddings-details} for details).
User-provider rewards are computed as $r(q_u, c) = q_u^Tc$.

\begin{table}[ht!]
\resizebox{\columnwidth}{!}{
\begin{tabular}{llll}
\toprule
    Data Type   & Method & Avg. Welfare & Viable Providers \\
\midrule
Uniform & Myopic &  -2.41 $\pm$ 0.59 &    43.80 $\pm$ 1.94 \\
 & LP-RS &  \textbf{-1.23 $\pm$ 0.57} & \textbf{48.00 $\pm$ 2.10} \\
Skewed & Myopic & -5.68 $\pm$ 0.61 &    34.40 $\pm$ 1.96 \\
 & LP-RS & \textbf{-3.40 $\pm$ 1.10} & \textbf{42.00 $\pm$ 3.35} \\
\bottomrule
\end{tabular}

}
\caption{
\label{tab:embedding-type}
(Synthetic data) Myopic recommendation performs poorly when observing \emph{skewed} user and provider embeddings (some topic space areas are more popular).
LP-RS improves social welfare \emph{and} number of viable providers for both data types.
Note that zero is the maximum possible welfare in this setting.
}
\end{table}

\subsection{Results}
\label{sec:results}

\newcommand{\subfigwidth}{.3\textwidth}

\begin{figure*}[t!]
  \centering
  \begin{subfigure}[t]{\subfigwidth}
    \noindent\resizebox{\textwidth}{!}{
      \includegraphics[width=\subfigwidth]{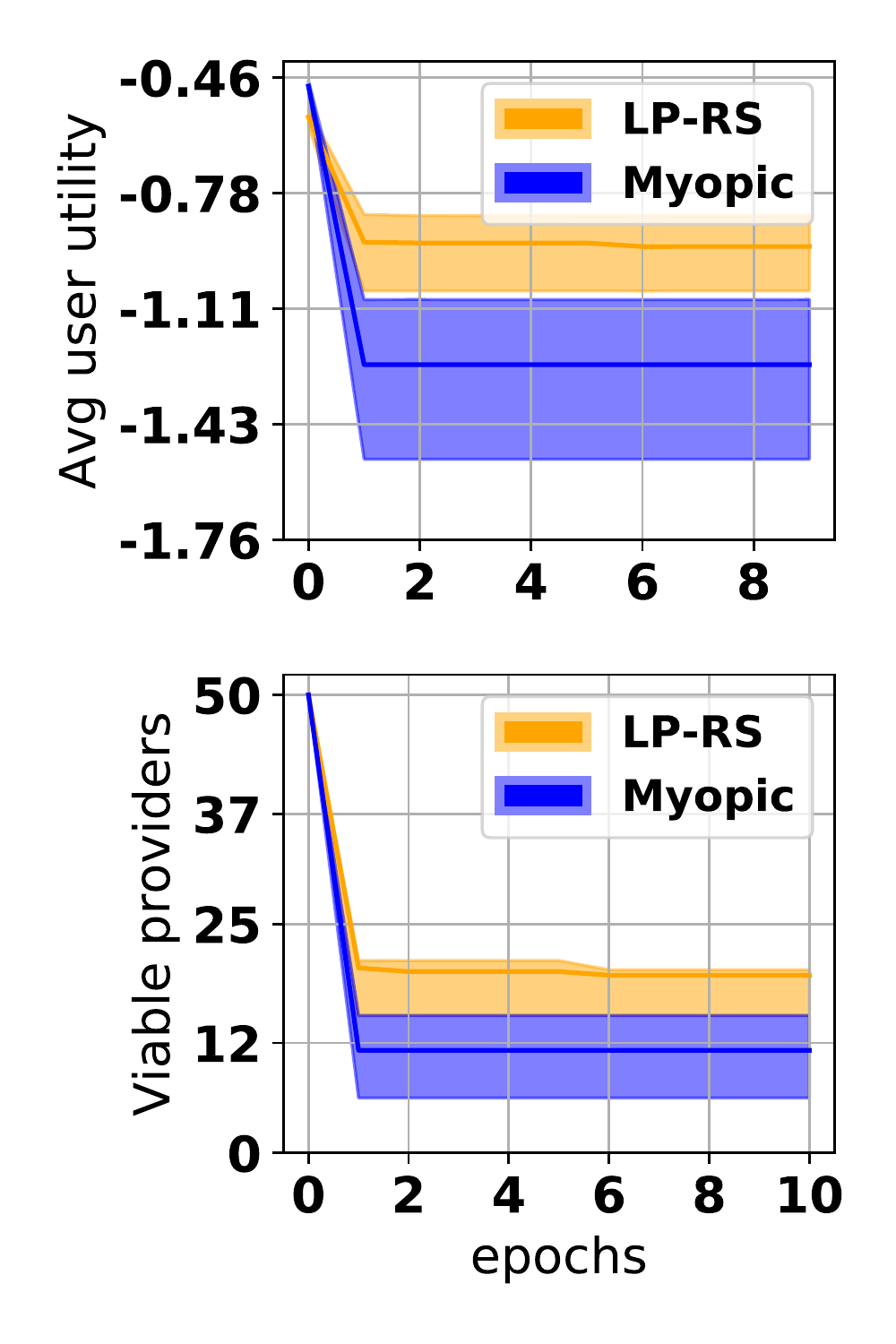}
    }
  \caption{
  Synthetic data embeddings.
  }\label{fig:synthetic-trajectories}
  \end{subfigure}%
  \begin{subfigure}[t]{\subfigwidth}
    \noindent\resizebox{\textwidth}{!}{
      \includegraphics[width=\subfigwidth]{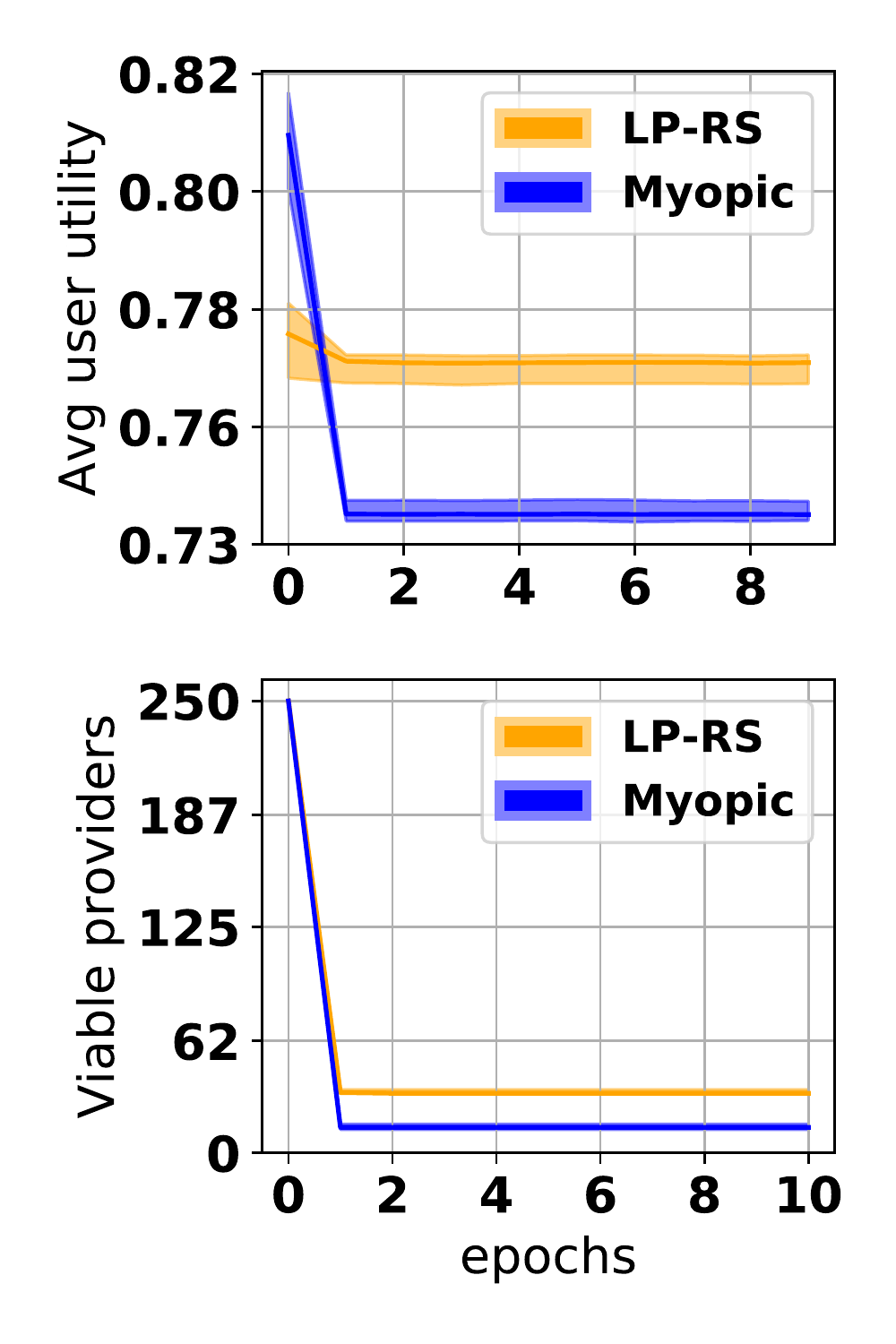}
    }
  \caption{
  Movielens data embeddings.
  }\label{fig:movielens-trajectories}
  \end{subfigure}%
  \begin{subfigure}[t]{\subfigwidth}
    \noindent\resizebox{\textwidth}{!}{
      \includegraphics[width=\subfigwidth]{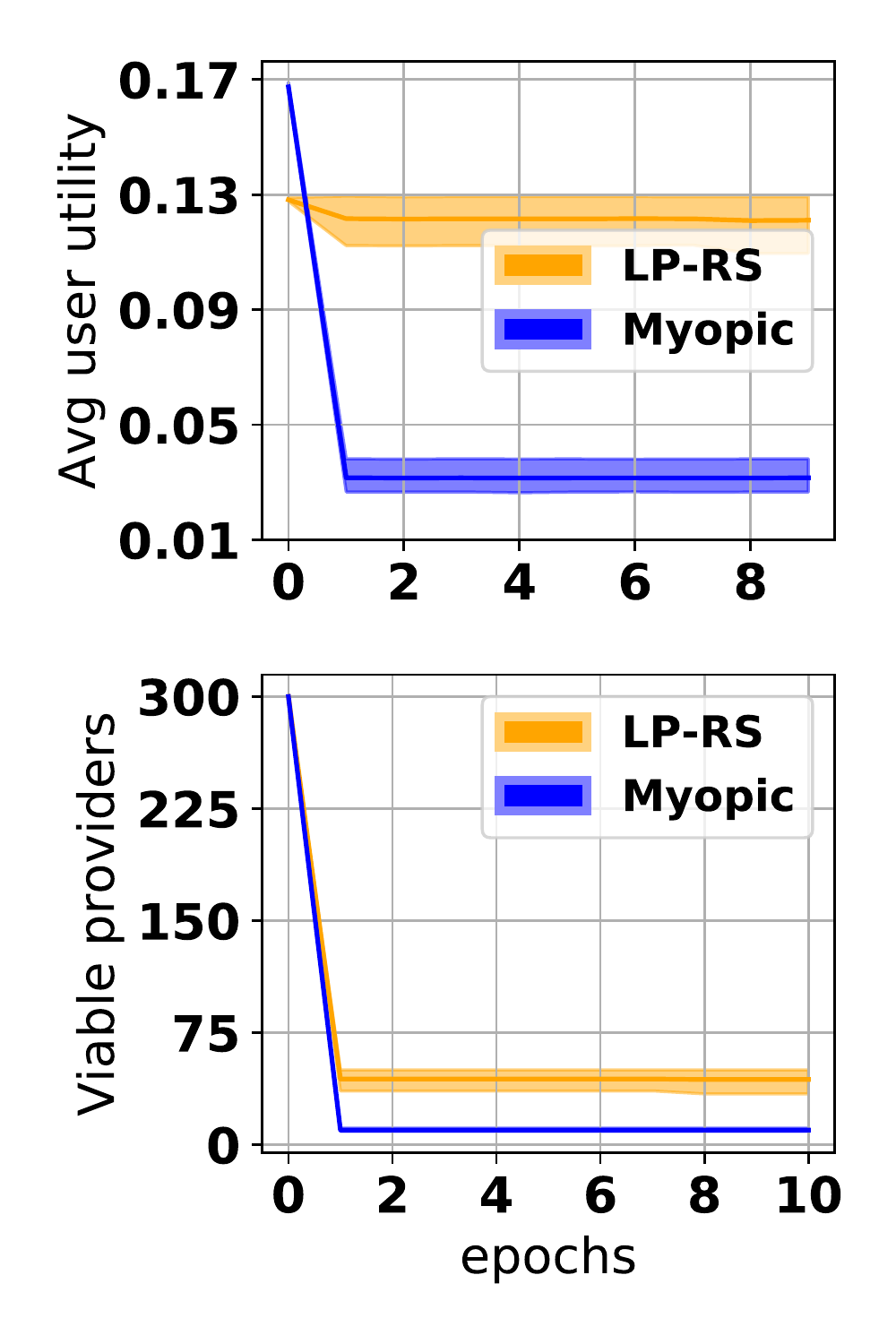} 
    }
  \caption{
  SNAP data embeddings.
  }\label{fig:snap-trajectories}
  \end{subfigure}%
  \caption{
    Average social welfare and number of viable providers per epoch of simulation. Bold lines are means over 5 seeds, while shaded regions show $25$-th to $75$-th percentiles across runs.
  }\label{fig:trajectories}
\end{figure*}

\paragraph{Exploring Embedding Type}
We begin by using synthetically generated embeddings to study how properties of the \emph{embeddings} affect long-term social welfare under $\pi_\text{My}$.
We evaluate $\pi_\text{LP}$ and $\pi_\text{My}$ using the uniform and skewed synthetic embeddings.
The results (Table \ref{tab:embedding-type}) show that when user/provider embeddings are \emph{skewed}, myopic recommendation yields suboptimal user welfare due to less popular providers abandoning the platform.
LP-RS improves welfare and increases the number
of viable providers in both cases.
For the remainder of the paper we use only the skewed variant of the synthetic data.

\newcommand{\histfigwidth}{.5\textwidth}
\begin{figure}[b!]
  \centering
    \noindent\resizebox{\histfigwidth}{!}{
    \hspace{-.5cm}
      \includegraphics[width=\histfigwidth]{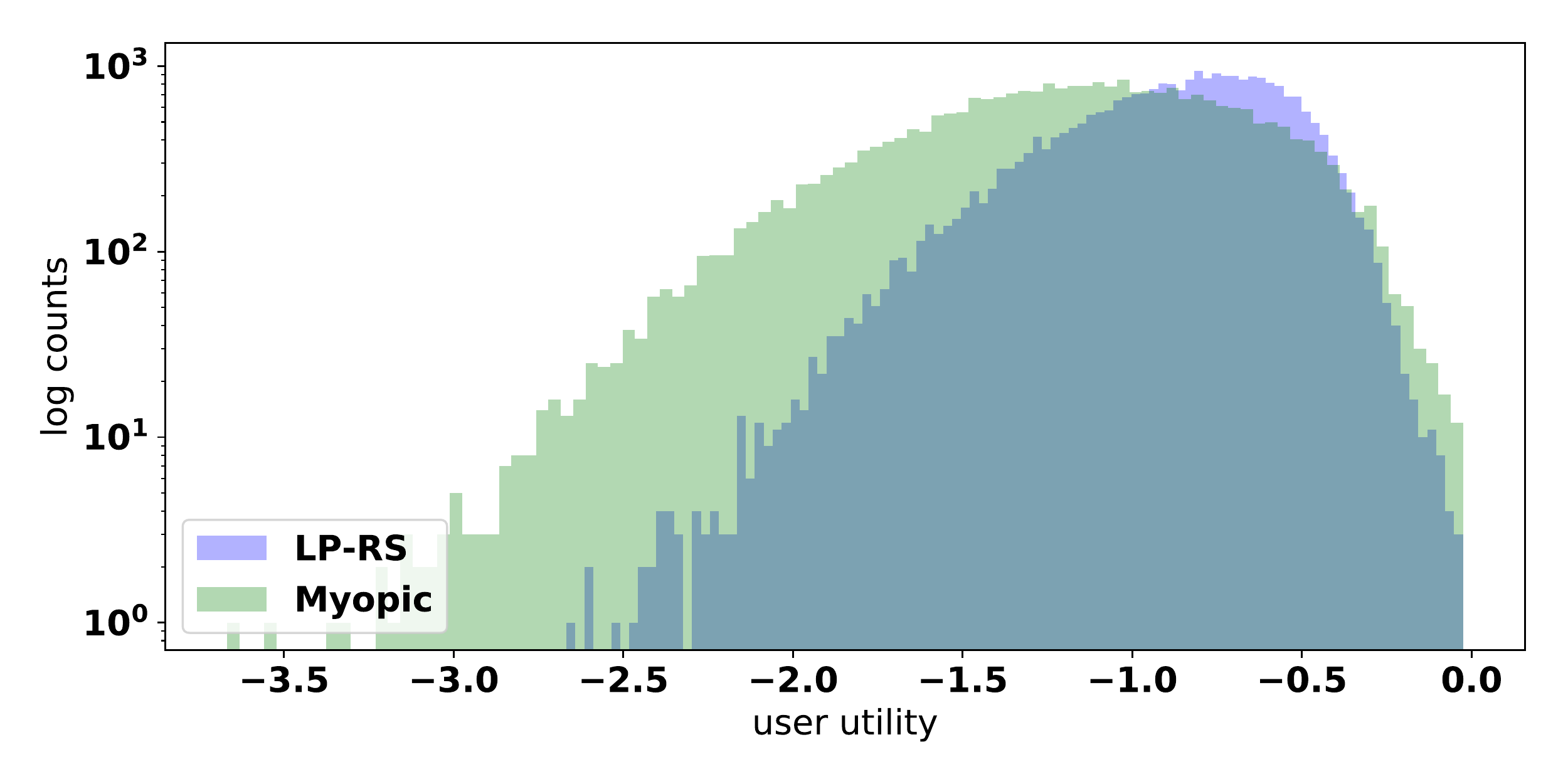}      
    }
  \caption{
    \label{fig:select_user_utility_histogram}
    User utility histogram (Synthetic embeddings).
    The LP recommender improves average social welfare at the expense of some maximum user regret.
  }
\end{figure}

\begin{table}[b!]
\resizebox{\columnwidth}{!}{
\begin{tabular}{lllll}
\toprule
$\gamma$ &      Avg. Welfare &        Max Regret & Regret-Welfare Ratio & Surviving Providers \\
\midrule
     0.1 &  18.02 $\pm$ 1.05 &   7.24 $\pm$ 0.77 &      0.40 $\pm$ 0.04 &    47.20 $\pm$ 1.72 \\
    0.18 &  19.79 $\pm$ 1.16 &   7.97 $\pm$ 0.84 &      0.40 $\pm$ 0.04 &    47.20 $\pm$ 1.72 \\
    0.26 &  21.87 $\pm$ 1.28 &   8.84 $\pm$ 0.91 &      0.41 $\pm$ 0.04 &    46.60 $\pm$ 2.15 \\
    0.35 &  24.30 $\pm$ 1.43 &  10.14 $\pm$ 1.18 &      0.42 $\pm$ 0.05 &    45.20 $\pm$ 3.06 \\
    0.43 &  27.17 $\pm$ 1.58 &  11.65 $\pm$ 0.91 &      0.43 $\pm$ 0.03 &    44.00 $\pm$ 3.16 \\
    0.51 &  30.44 $\pm$ 1.79 &  14.19 $\pm$ 1.46 &      0.47 $\pm$ 0.06 &    44.80 $\pm$ 0.98 \\
    0.59 &  34.30 $\pm$ 1.95 &  16.50 $\pm$ 1.73 &      0.48 $\pm$ 0.06 &    43.00 $\pm$ 1.90 \\
    0.67 &  38.67 $\pm$ 2.29 &  20.53 $\pm$ 1.46 &      0.53 $\pm$ 0.06 &    42.60 $\pm$ 1.50 \\
    0.75 &  43.69 $\pm$ 2.61 &  24.61 $\pm$ 1.31 &      0.57 $\pm$ 0.05 &    41.80 $\pm$ 1.72 \\
    0.84 &  49.51 $\pm$ 2.90 &  28.95 $\pm$ 1.11 &      0.59 $\pm$ 0.04 &    39.80 $\pm$ 3.06 \\
    0.92 &  56.15 $\pm$ 3.26 &  33.05 $\pm$ 1.16 &      0.59 $\pm$ 0.04 &    36.80 $\pm$ 3.97 \\
     1.0 &  63.62 $\pm$ 3.58 &  37.89 $\pm$ 1.39 &      0.60 $\pm$ 0.04 &    34.20 $\pm$ 5.08 \\
\bottomrule
\end{tabular}
}
\caption{
\label{tab:tradeoffs}
The trade-off between average user welfare and max user regret depends on the discounting factor $\gamma$.
The RS makes multiple recommendations to each user during each epoch, and a lower $\gamma$ indicates more steeply discounted returns for recommendations beyond the first one.
}
\end{table}

\paragraph{Tradeoffs in Regret and Welfare}
Next we investigate the trade-off between social welfare and individual regret induced by LP-RS at various levels of diminishing returns for user utility, introduced by discounting immediate user rewards as discussed in Sec~\ref{sec:simpleLP}.
Recall that $\Rgrt^\pi(u)$ is defined w.r.t.\ a policy $\pi^\ast_u$ that is 
``tailor made'' for user $u$, one that keeps all providers with
closest affinity to $u$ viable without regard to other users
(and will generally serve most users poorly). Under this definition,
every policy will generally have very high max regret. But this serves as a useful reference point for understanding how the preferences of any single user trade off with long-term
social welfare under a realistic policy $\pi \in \{\pi_\text{My}, \pi_\text{LP}\}$. We discuss the results for $\pi_\text{LP}$ in the following. Results for $\pi_\text{My}$
can be found in Appendix~\ref{sec:extra-table}. 

We expect that steeper rates of diminishing returns (lower $\gamma$) should lead to lower costs---that is, lower individual regret, or sacrifice of individual utility---to generate the optimal provider ``subsidies'' and generally induce more diverse provider sets in equlibrium, which in turn leads to lower max regret. This makes our social welfare objective more aligned with both provider and user fairness.  Table~\ref{tab:tradeoffs} corroborates this intuition, showing that better max-regret-to-average-welfare ratios are generally achieved at low levels of $\gamma$, as are a greater number of viable providers. 

\paragraph{Large-scale Simulations}
We carry out large-scale simulations using the dataset embeddings described in Section \ref{sec:datasets} (see Appendix~\ref{sec:simulation-details} for details).
Fig.~\ref{fig:trajectories} shows results from the simulations, tracking both the number of viable providers and average user utility (i.e., social welfare divided by number of users).
We find that $\pi_\text{LP}$ quickly converges to an equilibrium that sustains more providers than $\pi_\text{My}$.
Average user utility is also improved under $\pi_\text{LP}$.
Fig.~\ref{fig:select_user_utility_histogram} shows an example user utility histogram (aggregated over the entire simulation).
LP-RS has an overall positive impact on the distribution of user utility, relative to the myopic baseline.
However the increase in social welfare comes at the cost of decreased utility for
some of the most well-off users.
See Appendix~\ref{sec:more-histograms} for utility histograms for all datasets.

\section{Related Work}
\label{sec:related}
``Fairness'' in outcomes for users and providers could be described in a variety of (possibly conflicting) ways, and any computational measurement of fairness in this context will surely reflect normative principles underpinning the RS design \citep{binns2018fairness, leben2020normative}.
We have presented a utilitarian view whereby fair outcomes are realized when the average welfare of users is maximized in equilibrium, without taking other factors such as user/provider demographic group membership or diversity of content into account explicitly.
This is consistent with some recent approaches that model long-term fairness by constraining the exploration strategy of the decision maker \cite{joseph2016fairness,jabbari2017fairness} in the sense that fairness and optimality (here w.r.t. expected user rewards) are aligned.
However this approach may not always be appropriate, so we note that the general matching strategy we employ captures the \emph{dynamics} of the RS and could in principle be adapted to encourage other types of ``fair'' outcomes in accordance with different normative principles.

Studies of fairness in ranking typically consider notions of \emph{group fairness}, for example, by regularizing standard ranking metrics to encourage parity of the ranks across demographic groups \cite{Yang2016MeasuringFI,Zehlike2017FAIRAF}.
\citet{Celis2017RankingWF} propose a matching algorithm for ranking a set of items efficiently under a fairness constraint that encourages demographic diversity of items in the top-ranked position.

Research that considers fairness w.r.t. \emph{content providers} (rather than users), fairly allocating the exposure of providers to users \cite{SinghJoachims2018,Biega},
is closely related to our work, especially in some
aspects of its motivation.
These models impose fairness constraints that address ``provider'' concerns, and maximize user utility subject to these constraints. 
In this view, the policy makes commitments to every provider in the ecosystem, while user welfare is secondary to satisfying the fairness constraints. This requires that fairness constraints be crafted very carefully (which is a non-trivial problem \cite{Asudeh2019}) so as to not have an undue impact on user welfare. In our utilitarian framework, provider fairness is justified by user welfare, implicitly incorporating fairness constraints.

Also very related to our work is the work of 
\citet{benporat_etal:nips18}, who develop a game-theoretic
model of RSs whose providers act strategically---by making available
or withholding content---to maximize
the user engagement derived from an RS platform. They analyze
the policies of the RS: using an
axiomatic approach, they prove no RS policy can satisfy certain
properties jointly (including a form of fairness);
and in a cooperative game model, they show the uniqueness and
tractability of a simple allocation policy.
While related, their models are not dynamic and do not
(directly) assess complex user utility models; but they examine
more complex, \emph{strategic} behavior of providers in a way we do not.
\citet{benporat_etal:aaai19} draw a
connection between (strategic) facility
location games and RSs with strategic providers.
Our models relate to non-strategic facility location, with
an emphasis on scalable optimization methods.

\emph{Individual fairness} provides an important alternative perspective, requiring
that two individuals similar w.r.t.\ a task should be classified 
(or otherwise treated) similarly by the ML system \cite{Dwork}. 
Our utilitarian approach guarantees that providers with sufficient impact on user welfare remain viable, regardless of whether their audience lies at the tail or head of content space. As shown in Sec.~\ref{sec:experiments}, this provides a dramatic improvement over myopic policies which tend to serve ``head'' users and providers disproportionately well. While maximizing welfare does not guarantee high individual utility, we generally expect high individual utility to emerge across the user population if utility functions exhibit diminishing returns. If the form of user utility precludes this, the objective can be augmented with a maximum individual regret term as discussed in Sec~\ref{sec:regretopt}.

Also relevant is recent research extending
algorithmic fairness to dynamical systems.
Several methods for improved fairness in sequential decision-making have been proposed, including work on bandits
\citep{joseph2016fairness}, RL \citep{jabbari2017fairness},
and importance sampling \citep{doroudi2017importance}.
Fairness in dynamical systems has been explored in specific domains: 
predictive policing \citep{lum2016predict, ensign2018runaway}, hiring \citep{hu2018short,hu2019disparate}, 
lending \citep{mouzannar2019fair},
and RSs \cite{ChaneyStewartEngelhardt,bountouridis2019siren}.
In general, this work has focused on the role of algorithms
 in shaping environments over time
\citep{hashimoto2018fairness,kannan2019downstream},
observing that the repeated application of
algorithms in a changing environment impacts fairness in the long-term differently
from short-term effects.

\section{Conclusion}\label{sec:conclude}
We have developed a stylized model of the ecosystem
dynamics of a content recommender system. We have used it to study the effects of
typical myopic RS policies on content providers whose viability
depends on attaining a certain level of user engagement. We showed
that myopic policies can serve users poorly by driving
the system to an equilibrium in which many providers fail to remain viable,
inducing poor long-term (user) social welfare. By formulating the recommendation problem holistically as
an optimal constrained matching, these deficiencies can be overcome: we optimize long-term social welfare, while at the same time increasing \emph{provider} viability despite the fact that our objective is to increase \emph{user} welfare. We developed several
algorithmic approaches to the matching problem and experiments
with our LP-based approach showed significant improvements in
user welfare over myopic policies.
While our model is stylized, we believe it offers insights into
more general, realistic RS model as outlined in Sec.~\ref{sec:robust}. It provides a rich framework for studying
tradeoffs between individual utility and (utilitarian) social
welfare. Extensions to account for group fairness, strategic
behavior and exploration policies are critical areas of future
research as are new algorithmic techniques (e.g., reinforcement learning, online matching) as discussed in Sec.\ref{sec:robust}.

\section*{Acknowledgements}

Thanks to Francois Belletti, Yi-fan Chen and Aranyak Mehta for valuable discussions on this topic and to the reviewers for their helpful suggestions.

\bibliographystyle{icml2020}
\bibliography{long,paper}

\begin{thebibliography}{49}
\providecommand{\natexlab}[1]{#1}
\providecommand{\url}[1]{\texttt{#1}}
\expandafter\ifx\csname urlstyle\endcsname\relax
  \providecommand{\doi}[1]{doi: #1}\else
  \providecommand{\doi}{doi: \begingroup \urlstyle{rm}\Url}\fi

\bibitem[An et~al.(2017)An, Singh, and Svensson]{chan17facility}
An, H., Singh, M., and Svensson, O.
\newblock {LP}-based algorithms for capacitated facility location.
\newblock \emph{SIAM Journal on Computing}, 46\penalty0 (1):\penalty0 272--306,
  1 2017.
\newblock ISSN 0097-5397.
\newblock \doi{10.1137/151002320}.

\bibitem[Asudeh et~al.(2019)Asudeh, Jagadishy, Stoyanovichz, and
  Das]{Asudeh2019}
Asudeh, A., Jagadishy, H., Stoyanovichz, J., and Das, G.
\newblock Designing fair ranking schemes.
\newblock \emph{ACM SIGMOD Record}, 01 2019.

\bibitem[Barocas \& Selbst(2016)Barocas and Selbst]{barocasBigData}
Barocas, S. and Selbst, A.~D.
\newblock Big data's disparate impact.
\newblock \emph{California Law Review}, 671, 2016.

\bibitem[Ben{-}Porat \& Tennenholtz(2018)Ben{-}Porat and
  Tennenholtz]{benporat_etal:nips18}
Ben{-}Porat, O. and Tennenholtz, M.
\newblock A game-theoretic approach to recommendation systems with strategic
  content providers.
\newblock In \emph{Advances in Neural Information Processing Systems 31
  (NeurIPS-18)}, pp.\  1118--1128, Montreal, 2018.

\bibitem[Ben{-}Porat et~al.(2019)Ben{-}Porat, Goren, Rosenberg, and
  Tennenholtz]{benporat_etal:aaai19}
Ben{-}Porat, O., Goren, G., Rosenberg, I., and Tennenholtz, M.
\newblock From recommendation systems to facility location games.
\newblock In \emph{Proceedings of the Thirty-third {AAAI} Conference on
  Artificial Intelligence (AAAI-19)}, pp.\  1772--1779, Honolulu, 2019.

\bibitem[Beutel et~al.(2018)Beutel, Covington, Jain, Xu, Li, Gatto, and
  Chi]{beutel_etal:wsdm18}
Beutel, A., Covington, P., Jain, S., Xu, C., Li, J., Gatto, V., and Chi, E.~H.
\newblock Latent cross: Making use of context in recurrent recommender systems.
\newblock In \emph{Proceedings of the Eleventh {ACM} International Conference
  on Web Search and Data Mining (WSDM-18)}, pp.\  46--54, Marina Del Rey, CA,
  2018.

\bibitem[Beutel et~al.(2019)Beutel, Chen, Doshi, Qian, Wei, Wu, Heldt, Zhao,
  Hong, Chi, and Goodrow]{beutel_etal:kdd19}
Beutel, A., Chen, J., Doshi, T., Qian, H., Wei, L., Wu, Y., Heldt, L., Zhao,
  Z., Hong, L., Chi, E.~H., and Goodrow, C.
\newblock Fairness in recommendation ranking through pairwise comparisons.
\newblock In \emph{Proceedings of the 25th {ACM} {SIGKDD} International
  Conference on Knowledge Discovery {\&} Data Mining (KDD-19)}, pp.\
  2212--2220, Anchorage, AK, 2019.

\bibitem[Biega et~al.(2018)Biega, Gummadi, and Weikum]{Biega}
Biega, A.~J., Gummadi, K.~P., and Weikum, G.
\newblock Equity of attention: Amortizing individual fairness in rankings.
\newblock In \emph{The 41st International {ACM SIGIR} Conference on Research \&
  Development in Information Retrieval}, SIGIR ’18, pp.\  405–414, New
  York, NY, USA, 2018. Association for Computing Machinery.

\bibitem[Binns(2018)]{binns2018fairness}
Binns, R.
\newblock Fairness in machine learning: Lessons from political philosophy.
\newblock In \emph{Conference on Fairness, Accountability and Transparency},
  pp.\  149--159, 2018.

\bibitem[Bountouridis et~al.(2019)Bountouridis, Harambam, Makhortykh, Marrero,
  Tintarev, and Hauff]{bountouridis2019siren}
Bountouridis, D., Harambam, J., Makhortykh, M., Marrero, M., Tintarev, N., and
  Hauff, C.
\newblock {SIREN}: A simulation framework for understanding the effects of
  recommender systems in online news environments.
\newblock In \emph{Proceedings of the Conference on Fairness, Accountability,
  and Transparency}, pp.\  150--159. ACM, 2019.

\bibitem[Boyd et~al.(2011)Boyd, Parikh, Chu, Peleato, and Eckstein]{boyd_admm}
Boyd, S., Parikh, N., Chu, E., Peleato, B., and Eckstein, J.
\newblock Distributed optimization and statistical learning via the alternating
  direction method of multipliers.
\newblock \emph{Found. Trends Mach. Learn.}, 3\penalty0 (1):\penalty0 1–122,
  January 2011.
\newblock ISSN 1935-8237.
\newblock \doi{10.1561/2200000016}.
\newblock URL \url{https://doi.org/10.1561/2200000016}.

\bibitem[Celis et~al.(2017)Celis, Straszak, and Vishnoi]{Celis2017RankingWF}
Celis, L.~E., Straszak, D., and Vishnoi, N.~K.
\newblock Ranking with fairness constraints.
\newblock In \emph{ICALP}, 2017.

\bibitem[Celma(2010)]{celma:longtail2010}
Celma, {\`O}.
\newblock The long tail in recommender systems.
\newblock In \emph{Music Recommendation and Discovery}, pp.\  87--107.
  Springer, 2010.

\bibitem[Chaney et~al.(2018)Chaney, Stewart, and
  Engelhardt]{ChaneyStewartEngelhardt}
Chaney, A. J.~B., Stewart, B.~M., and Engelhardt, B.~E.
\newblock How algorithmic confounding in recommendation systems increases
  homogeneity and decreases utility.
\newblock In \emph{Proceedings of the 12th ACM Conference on Recommender
  Systems}, RecSys ’18, pp.\  224–232, New York, NY, USA, 2018. Association
  for Computing Machinery.

\bibitem[Chen et~al.(2018)Chen, Beutel, Covington, Jain, Belletti, and
  Chi]{chen_etal:2018top}
Chen, M., Beutel, A., Covington, P., Jain, S., Belletti, F., and Chi, E.
\newblock Top-k off-policy correction for a {REINFORCE} recommender system.
\newblock In \emph{12th {ACM} International Conference on Web Search and Data
  Mining (WSDM-19)}, pp.\  456--464, Melbourne, Australia, 2018.

\bibitem[Cornuejols et~al.(1977)Cornuejols, Fisher, and
  Nemhauser]{nemhauser77Facility}
Cornuejols, G., Fisher, M.~L., and Nemhauser, G.~L.
\newblock Location of bank accounts to optimize float: An analytic study of
  exact and approximate algorithms.
\newblock \emph{Management Science}, 23\penalty0 (8):\penalty0 789--810, 1977.
\newblock URL \url{http://www.jstor.org/stable/2630709}.

\bibitem[Covington et~al.(2016)Covington, Adams, and
  Sargin]{covington:recsys16}
Covington, P., Adams, J., and Sargin, E.
\newblock Deep neural networks for {YouTube} recommendations.
\newblock In \emph{Proceedings of the 10th {ACM} Conference on Recommender
  Systems}, pp.\  191--198, Boston, 2016.

\bibitem[Doroudi et~al.(2017)Doroudi, Thomas, and
  Brunskill]{doroudi2017importance}
Doroudi, S., Thomas, P.~S., and Brunskill, E.
\newblock Importance sampling for fair policy selection.
\newblock In \emph{Uncertainty in Artificial Intelligence (UAI)}, 2017.

\bibitem[Dwork et~al.(2012)Dwork, Hardt, Pitassi, Reingold, and Zemel]{Dwork}
Dwork, C., Hardt, M., Pitassi, T., Reingold, O., and Zemel, R.
\newblock Fairness through awareness.
\newblock In \emph{Proceedings of the 3rd Innovations in Theoretical Computer
  Science Conference}, ITCS ’12, pp.\  214–226. Association for Computing
  Machinery, 2012.

\bibitem[Ensign et~al.(2018)Ensign, Friedler, Neville, Scheidegger, and
  Venkatasubramanian]{ensign2018runaway}
Ensign, D., Friedler, S.~A., Neville, S., Scheidegger, C., and
  Venkatasubramanian, S.
\newblock Runaway feedback loops in predictive policing.
\newblock In \emph{Conference on Fairness, Accountability and Transparency},
  pp.\  160--171, 2018.

\bibitem[Hajiaghayi et~al.(2003)Hajiaghayi, Mahdian, and
  Mirrokni]{vahab03facility}
Hajiaghayi, M.~T., Mahdian, M., and Mirrokni, V.~S.
\newblock The facility location problem with general cost functions.
\newblock \emph{Networks}, 42\penalty0 (1):\penalty0 42--47, 2003.

\bibitem[Harper \& Konstan(2015)Harper and Konstan]{harper2015movielens}
Harper, F.~M. and Konstan, J.~A.
\newblock The {Movielens} datasets: History and context.
\newblock \emph{{ACM} Transactions on Interactive Intelligent Systems},
  5\penalty0 (4):\penalty0 1--19, 2015.

\bibitem[Hashimoto et~al.(2018)Hashimoto, Srivastava, Namkoong, and
  Liang]{hashimoto2018fairness}
Hashimoto, T.~B., Srivastava, M., Namkoong, H., and Liang, P.
\newblock Fairness without demographics in repeated loss minimization.
\newblock In \emph{International Conference on Machine Learning}, 2018.

\bibitem[He et~al.(2017)He, Liao, Zhang, Nie, Hu, and Chua]{he_etal:www17}
He, X., Liao, L., Zhang, H., Nie, L., Hu, X., and Chua, T.
\newblock Neural collaborative filtering.
\newblock In \emph{Proceedings of the 26th International Conference on World
  Wide Web (WWW-17)}, pp.\  173--182, Perth, Australia, 2017.

\bibitem[Hu \& Chen(2018)Hu and Chen]{hu2018short}
Hu, L. and Chen, Y.
\newblock A short-term intervention for long-term fairness in the labor market.
\newblock In \emph{Proceedings of the 2018 World Wide Web Conference}, pp.\
  1389--1398. International World Wide Web Conferences Steering Committee,
  2018.

\bibitem[Hu et~al.(2019)Hu, Immorlica, and Vaughan]{hu2019disparate}
Hu, L., Immorlica, N., and Vaughan, J.~W.
\newblock The disparate effects of strategic manipulation.
\newblock In \emph{Proceedings of the Conference on Fairness, Accountability,
  and Transparency}, pp.\  259--268. ACM, 2019.

\bibitem[Hu et~al.(2008)Hu, Koren, and Volinsky]{hu2008collaborative}
Hu, Y., Koren, Y., and Volinsky, C.
\newblock Collaborative filtering for implicit feedback datasets.
\newblock In \emph{2008 Eighth IEEE International Conference on Data Mining},
  pp.\  263--272. Ieee, 2008.

\bibitem[Ie et~al.(2019{\natexlab{a}})Ie, Hsu, Mladenov, Jain, Narvekar, Wang,
  Wu, and Boutilier]{ie2019recsim}
Ie, E., Hsu, C.-w., Mladenov, M., Jain, V., Narvekar, S., Wang, J., Wu, R., and
  Boutilier, C.
\newblock Recsim: A configurable simulation platform for recommender systems.
\newblock \emph{arXiv preprint arXiv:1909.04847}, 2019{\natexlab{a}}.

\bibitem[Ie et~al.(2019{\natexlab{b}})Ie, Jain, Wang, Narvekar, Agarwal, Wu,
  Cheng, Chandra, and Boutilier]{slateQ:ijcai19}
Ie, E., Jain, V., Wang, J., Narvekar, S., Agarwal, R., Wu, R., Cheng, H.-T.,
  Chandra, T., and Boutilier, C.
\newblock {SlateQ}: A tractable decomposition for reinforcement learning with
  recommendation sets.
\newblock In \emph{Proceedings of the Twenty-eighth International Joint
  Conference on Artificial Intelligence (IJCAI-19)}, pp.\  2592--2599, Macau,
  2019{\natexlab{b}}.

\bibitem[Jabbari et~al.(2017)Jabbari, Joseph, Kearns, Morgenstern, and
  Roth]{jabbari2017fairness}
Jabbari, S., Joseph, M., Kearns, M., Morgenstern, J., and Roth, A.
\newblock Fairness in reinforcement learning.
\newblock In \emph{Proceedings of the 34th International Conference on Machine
  Learning-Volume 70}, pp.\  1617--1626. JMLR.org, 2017.

\bibitem[Jacobson et~al.(2016)Jacobson, Murali, Newett, Whitman, and
  Yon]{jacobson2016music}
Jacobson, K., Murali, V., Newett, E., Whitman, B., and Yon, R.
\newblock Music personalization at {Spotify}.
\newblock In \emph{Proceedings of the 10th {ACM} Conference on Recommender
  Systems (RecSys16)}, pp.\  373--373, Boston, Massachusetts, USA, 2016.

\bibitem[Jones(2015)]{jones15facility}
Jones, M.
\newblock The maximum facility location problem.
\newblock B.s. thesis, The University of Sydney, Sydney, Australia, 2015.

\bibitem[Joseph et~al.(2016)Joseph, Kearns, Morgenstern, and
  Roth]{joseph2016fairness}
Joseph, M., Kearns, M., Morgenstern, J.~H., and Roth, A.
\newblock Fairness in learning: Classic and contextual bandits.
\newblock In \emph{Advances in Neural Information Processing Systems}, pp.\
  325--333, 2016.

\bibitem[Kahneman \& Tversky(1979)Kahneman and Tversky]{kahneman_prospect:1979}
Kahneman, D. and Tversky, A.
\newblock Prospect theory: An analysis of decision under risk.
\newblock \emph{Econometrica}, 47\penalty0 (2):\penalty0 263--292, 1979.

\bibitem[Kannan et~al.(2019)Kannan, Roth, and Ziani]{kannan2019downstream}
Kannan, S., Roth, A., and Ziani, J.
\newblock Downstream effects of affirmative action.
\newblock In \emph{Proceedings of the Conference on Fairness, Accountability,
  and Transparency}, pp.\  240--248. ACM, 2019.

\bibitem[Konstan et~al.(1997)Konstan, Miller, Maltz, Herlocker, Gordon, and
  Riedl]{grouplens:cacm97}
Konstan, J.~A., Miller, B.~N., Maltz, D., Herlocker, J.~L., Gordon, L.~R., and
  Riedl, J.
\newblock {GroupLens}: Applying collaborative filtering to usenet news.
\newblock \emph{Communications of the {ACM}}, 40\penalty0 (3):\penalty0 77--87,
  1997.

\bibitem[Kwak et~al.(2010)Kwak, Lee, Park, and Moon]{kwak2010twitter}
Kwak, H., Lee, C., Park, H., and Moon, S.
\newblock What is {Twitter}, a social network or a news media?
\newblock In \emph{Proceedings of the 19th International Conference on {W}orld
  {W}ide {W}eb}, pp.\  591--600, 2010.

\bibitem[Leben(2020)]{leben2020normative}
Leben, D.
\newblock Normative principles for evaluating fairness in machine learning.
\newblock In \emph{Proceedings of the AAAI/ACM Conference on AI, Ethics, and
  Society}, AIES ’20, pp.\  86–92, New York, NY, USA, 2020. Association for
  Computing Machinery.
\newblock ISBN 9781450371100.
\newblock \doi{10.1145/3375627.3375808}.
\newblock URL \url{https://doi.org/10.1145/3375627.3375808}.

\bibitem[Li(2019)]{li19Facility}
Li, S.
\newblock On facility location with general lower bounds.
\newblock In \emph{Proceedings of the Thirtieth Annual ACM-SIAM Symposium on
  Discrete Algorithms}, SODA ’19, pp.\  2279–2290, USA, 2019. Society for
  Industrial and Applied Mathematics.

\bibitem[Lum \& Isaac(2016)Lum and Isaac]{lum2016predict}
Lum, K. and Isaac, W.
\newblock To predict and serve?
\newblock \emph{Significance}, 13\penalty0 (5):\penalty0 14--19, 2016.

\bibitem[Mehta(2013)]{mehta:onlinematching2013}
Mehta, A.
\newblock Online matching and ad allocation.
\newblock \emph{Foundations and Trends in Theoretical Computer Science},
  8\penalty0 (4):\penalty0 265--368, 2013.

\bibitem[Mirzasoleiman et~al.(2016)Mirzasoleiman, Karbasi, Sarkar, and
  Krause]{JMLR:v17:mirzasoleiman16a}
Mirzasoleiman, B., Karbasi, A., Sarkar, R., and Krause, A.
\newblock Distributed submodular maximization.
\newblock \emph{Journal of Machine Learning Research}, 17\penalty0
  (235):\penalty0 1--44, 2016.
\newblock URL \url{http://jmlr.org/papers/v17/mirzasoleiman16a.html}.

\bibitem[Mladenov et~al.(2020)Mladenov, Creager, Ben-Porat, Swersky, Zemel, and
  Boutilier]{fullversion}
Mladenov, M., Creager, E., Ben-Porat, O., Swersky, K., Zemel, R., and
  Boutilier, C.
\newblock Optimizing long-term social welfare in recommender systems: A
  constrained matching approach.
\newblock Technical report, 2020.
\newblock arXiv:2008.00104.

\bibitem[Mouzannar et~al.(2019)Mouzannar, Ohannessian, and
  Srebro]{mouzannar2019fair}
Mouzannar, H., Ohannessian, M.~I., and Srebro, N.
\newblock From fair decision making to social qquality.
\newblock In \emph{Proceedings of the Conference on Fairness, Accountability,
  and Transparency}, pp.\  359--368. ACM, 2019.

\bibitem[Ribeiro et~al.(2020)Ribeiro, Ottoni, West, Almeida, and
  Jr.]{Ribeiro_etal:FAT20}
Ribeiro, M.~H., Ottoni, R., West, R., Almeida, V. A.~F., and Jr., W.~M.
\newblock Auditing radicalization pathways on {YouTube}.
\newblock In \emph{FAT* '20: Conference on Fairness, Accountability, and
  Transparency}, pp.\  131--141, Barcelona, 2020.

\bibitem[Salakhutdinov \& Mnih(2007)Salakhutdinov and
  Mnih]{salakhutdinov-mnih:nips07}
Salakhutdinov, R. and Mnih, A.
\newblock Probabilistic matrix factorization.
\newblock In \emph{Advances in Neural Information Processing Systems 20
  (NIPS-07)}, pp.\  1257--1264, Vancouver, 2007.

\bibitem[Singh \& Joachims(2018)Singh and Joachims]{SinghJoachims2018}
Singh, A. and Joachims, T.
\newblock Fairness of exposure in rankings.
\newblock In \emph{Proceedings of the 24th {ACM} {SIGKDD} International
  Conference on Knowledge Discovery \& Data Mining}, KDD ’18, pp.\
  2219–2228, New York, NY, USA, 2018. Association for Computing Machinery.

\bibitem[Yang \& Stoyanovich(2016)Yang and Stoyanovich]{Yang2016MeasuringFI}
Yang, K. and Stoyanovich, J.
\newblock Measuring fairness in ranked outputs.
\newblock In \emph{SSDBM '17}, 2016.

\bibitem[Zehlike et~al.(2017)Zehlike, Bonchi, Castillo, Hajian, Megahed, and
  Baeza-Yates]{Zehlike2017FAIRAF}
Zehlike, M., Bonchi, F., Castillo, C., Hajian, S., Megahed, M., and
  Baeza-Yates, R.
\newblock Fa*ir: A fair top-k ranking algorithm.
\newblock In \emph{Proceedings of the 2017 ACM on Conference on Information and
  Knowledge Management}, CIKM ’17, pp.\  1569–1578, New York, NY, USA,
  2017. Association for Computing Machinery.

\end{thebibliography}

\newpage
\onecolumn
\appendix
\section{Algorithms and Proofs}
\subsection{Greedy Optimization and Theorem \ref{thm:additive}}
\label{sec:proof_additive}
\subsubsection{Preliminaries}
Before we begin the proof, we make a few notational modifications to significantly simplify it.  Let ${\cal C}$ be a set of content providers (hereinafter providers), ${\cal U}$ a set of users, and $A\in{\mathbb R}^{|{\cal U}|\times |{\cal C }|}$ a utility matrix. Furthermore, let $D: {\cal U} \rightarrow {\mathbb N}$ be a user demand function (specifying how many queries a user $u$ submits to the system) and $\nu_c$ for $c\in \cal C$ be the provider survival threshold, indicating how many queries the provider needs to receive in order to be viable. For this section, we will make the following simplifying assumptions:
\begin{itemize}
    \item every user has exactly one unique query during the epoch, and
    \item every user's view contributes exactly one unit towards a provider's viability.
\end{itemize}
Under these assumptions, the set of queries becomes identical to the set of users ($\calU = \calQ$), and $\bar{Q}(q_u) = 1$. We  proceed to prove the submodularity of user welfare as a function of the provider set subject to the above restrictions. After that, we discuss the reduction of Problem~\ref{eq:ilp_linear_util} to this restricted case. 

The welfare maximization problem is then to find a matching such that
\begin{align}
\label{eq:general_problem}
X^\ast =  \arg\max_{X,Y} & \sum_{u \in {\cal U}} \left(\sum_{t=1}^{D(u)}\sum_{c\in{\cal C}} A_{uc}X_{uct}  \right)\nonumber\\ 
\text{subject to }& \sum_{c \in {\cal C}} X_{uct} = 1\quad \forall u\in{\cal U}, t\in \{1,\ldots, D(u)\}\nonumber\\
& X_{uct} \leq Y_c\quad \forall u\in{\cal U}, c\in{\cal C}\nonumber\\
& \sum_{u\in{\cal U}}\sum_{t=1}^{D(u)} X_{uct} \geq \nu_cY_c, \quad\forall c\in {\cal C}\nonumber\\
& X_{uct}, Y_c \in \{0,1\}, \quad \forall u\in{\cal U}, c\in{\cal C},t\in [1,\ldots, D(u)] ,
\end{align}
The problem~\eqref{eq:general_problem} is a hard combinatorial problem, so the question is if we can derive good heuristics for solving it. Of particular interest is the following greedy heuristic: let $C\subseteq {\cal C}$ and define $g: 2^{\cal C} \rightarrow \mathbb{R}$ as 
\begin{align}
\label{eq:local_problem}
    g(C) \mapsto \max_X &  \sum_{u \in {\cal U}} \left(\sum_{t=1}^{D(u)}\sum_{c\in C} A_{uc}X_{uct}  \right)\nonumber\\
    \text{subject to }& \sum_{c\in C} X_{uct} = 1\quad \forall u\in{\cal U}, t\in [1,\ldots, D(u)]\nonumber\\
   & \sum_{u\in{\cal U}}\sum_{t=1}^{D(u)} X_{uct} \geq \nu_c \quad\forall c\in C \nonumber\\
   & X_{uct} \in \{0,1\}, \quad \forall u\in{\cal U}, c\in{\cal C},t\in [1,\ldots, D(u)].
\end{align}
That is, $g(C)$ is the best matching if the provider set $C\subset \mathcal C$ is fixed externally. Despite that (\ref{eq:local_problem}) has binary constraints on $X_{uct}$, its constraint matrix is Totally Unimodular; hence, we are guaranteed that (\ref{eq:local_problem}) is integral. The goal is then to start with $C = \emptyset$ and greedily add providers while $g(C)$ keeps improving.
In order for this to work well, $g$ would need to be sub-modular, which is precisely what we prove next.
\begin{theorem}\label{thm:additive}
For every two providers $c_0,c_1\in \mathcal C$ and $C\subseteq \mathcal C\setminus\{c_0,c_1\}$, it holds that
\begin{equation}\label{eq:submodinthm}
g(C\cup\{c_0, c_1\}) - g(C\cup\{c_1\}) \leq g(C\cup\{c_0\}) - g(C).
\end{equation}
\end{theorem}

\subsubsection{Proof of Theorem \ref{thm:additive} in the Unit Case}
Let us make a simplification: a user with $D(u)$ queries is equivalent to $D(u)$ independent users; thus, we will just work with an extended user set. We now present the terminology used in this proof. A matching $X:\mU \rightarrow \mC,$ is a function from users to providers. We denote by $\mC(X)$ the serving providers under $X$, i.e., $\mC(X) =\{ c \mid \exists u\in \mU, X(u)=c\}$. Further, we say that a matching $X$ is \textit{feasible} if every provider in $\mC(X)$ meets her threshold under $X$, namely, if for every $c\in \mC(X)$ it holds that $\abs{\{u\in \mU \mid X(u)=c  \}}\geq \nu_c$. We denote by $F(X)$ the value obtained for a feasible matching $X$ in Problem \eqref{eq:local_problem} (note that $X$ may not be optimal w.r.t.. $\mC (X)$). In the rest of the proof, we rely on optimal matchings for $C, C\cup\{c_1\}$ and $C\cup\{c_0,c_1\}$ to construct a new matching, $X^0$. The active providers under $X^0$ are $C\cup \{c_0\}$ and, as we shall show, $X^0$ satisfies
\begin{equation}
\label{eq:resulting_x_prime}
g(C\cup\{c_0, c_1\}) - g(C\cup\{c_1\}) \leq F(X^0) - g(C).    
\end{equation}
The latter immediately implies Inequality \eqref{eq:submodinthm}, since by definition of $g$,
\[F(C\cup\{c_0, c_1\}) \leq \max_{X:\mC(X)=C\cup\{c_0\}} F(X)=g(C\cup \{c_0\}).\]

We are now ready to develop the tools required for the proof. 
The next notion assists to succinctly quantify the difference in user utility between two matchings.
\begin{definition}
Let $X$ be and $Y$ be two feasible matchings. We call a triplet $(c,c',u)$ a relocation triplet w.r.t. $X,Y$ if $X(u)=c$, $Y(u) = c'$ and $c\neq c'$.
\end{definition}
Importantly, two matchings define a unique set of (ordered) relocation triplets. conversely, a source matching and relocation triplets uniquely define the target matching.

Let $X$ and $X^1$ denote (any) optimal matching induced by $g(C)$ and $g({C\cup \{c_1\}})$ in Problem (\ref{eq:local_problem}), respectively. We now construct a graph whose nodes are the providers and its edges correspond to relocation triplets w.r.t. $X,X^1$. Formally, let $G^1=(\mC,E^1,w)$ denote a directed multi-graph, where the set of nodes is $\mC$; $E^1$ is the set of all relocation triplets  w.r.t. $X,X^1$, where every triplet $(c,c',u)$ forms a directed edge from $c=X(u)$ to $c'=X^1(u)$ with an ID of $u$; and the weight function $w$ is defined by $w(c,c',u)=A_{uc'}-A_{uc}$. Observe that the number of users each provider $c$ (a node in the graph) obtains under $X^1$ equals
\begin{equation}
\label{eq:num_of_users}
\abs{\{u\in \mU \mid X(u)=c\}}+\deg ^{+}(c)-\deg^{-}(c),    
\end{equation}
where $\deg^{+}(c)$ denotes the indegree of $c$ and its outdegree is denoted by $\deg^{-}(c)$. Moreover, the sum of weights is precisely the difference in utility between $X$ and $X^1$, i.e.,
\[
F(X^1)-F(X)=g(C\cup\{c_1\})-g(C) = \sum_{e\in E^1} w(e).
\]
In the next proposition, we use the fact that $X,X^1$ are optimal w.r.t. their provider sets to characterize properties of $G^1$.
\begin{proposition}\label{prop:no_cycle_and_sink}
It holds that:
\begin{enumerate}
    \item[(1)] $G^1$ does not contain directed cycles.
    \item[(2)] The only sink in $G^1$ is $c_1$.
\end{enumerate}
\end{proposition}
The proof of Proposition \ref{prop:no_cycle_and_sink} appears below. Proposition \ref{prop:no_cycle_and_sink} suggests that $G^1$ is a DAG with flow conservation, so we can decompose its edges into a set of independent paths (for any arbitrary partition into paths) between a source, i.e. a provider with an excess of users under the matching $X$, and the sink $c_1$.

Next, we introduce a second graph, $G^{0,1}$, with relocation triplets from $X^1$ to $X^{0,1}$, the optimal matching for $g(C\cup\{c_0,c_1\})$. Formally, $G^{0,1}=(\mC,E^{0,1},w)$ is a directed multi-graph, with the same set of nodes and the same weight function $w$. $E^{0,1}$ is composed of all relocation triplets from $X^1$ to $X^{0,1}$. By mirroring the proof of Proposition \ref{prop:no_cycle_and_sink}, we conclude that $G^{0,1}$ contains no cycles and that $c_0$ is the unique sink of every directed path in it. This graph is of special interest because its sum of weights is the left hand side of Inequality \eqref{eq:submodinthm}. Namely, $\sum_{e\in E^{0,1}}w(e)=F(X^{0,1})-F(X^1)=g(C\cup \{c_0,c_1\})-g(C\cup\{c_1\})$. It also describes how to optimally relocate users from $C\cup \{c_1\}$ to $C\cup \{c_0,c_1\}$. 

After understanding the structural properties of $G^1$ and $G^{0,1}$, we are ready to construct the promised matching $X^0$ (recall Inequality \eqref{eq:resulting_x_prime}). Let $G=(\mC, E^1 \cup E^{0,1},w)$ be the graph on the same set of nodes $\mC$, with all the edges from both $E^1$ and $E^{0,1}$ (notice that the same edge cannot appear in both).  For simplicity, we refer to paths in $E^1$ as \textit{blue} and to paths in $E^{0,1}$ as \textit{red}, for some arbitrary partition into paths. Our goal is to select a subset $E$ of edges from $E^1 \cup E^{0,1}$, which, when applied to $X$, will induce the matching $X^0$. To that end, we devise an iterative process to construct the set $E$, by adding one path at the time. The key property of this process, which we formalize via Algorithm \ref{alg:green_flow}, is that there exists a mapping from every red path to a new path, composed of red and (potentially) blue edges, with a less or equal weight than that red path.

To illustrate why this process is necessary, observe that not every subset of $E^1 \cup E^{0,1}$ can be applied to $X$ in order to obtain a new valid matching. In particular, recall that $E^{0,1}$ is the difference between $X^1$ and $X^{0,1}$; thus, a red path may involve the relocation $(c, c^\prime, u)$, where $u$ might have been matched to $c^\prime$ due some blue relocation $(c^{\prime\prime}, c, u)$. To ensure that the subset we pick will result in a valid matching, we make the following distinction: a subset $E$ such that  $E\subseteq E^1 \cup E^{0,1}$ is called \emph{consistent} if for any relocation triplet $(c,c',u)\in E$ either $X(u)=c$ or there exists another relocation triplet $(c'',c,u)\in E$. Informally, $E$ is consistent if every user $u$ that was relocated to $c'$ from $c$ was either matched to $c$ in $X$, or was relocated to $c$ from another provider. Consistency of the relocation triplets is a necessary, but not a sufficient condition for the resulting matching to be feasible.

Another useful notion is that of a \emph{junction node}. We say that a node $c\in \mathcal C$ is a junction w.r.t. $E^1, E^{0,1}$ if there exists a blue edge $(c'',c,u)\in E^1$ and a red edge $(c,c',u)\in E^{0,1}$ for some $c',c''\in \mathcal C$ and $u\in \mathcal U$. See Fig. \ref{fig:intersection} for illustration.

\begin{figure}[h!]
\centering
\includegraphics[scale=.17]{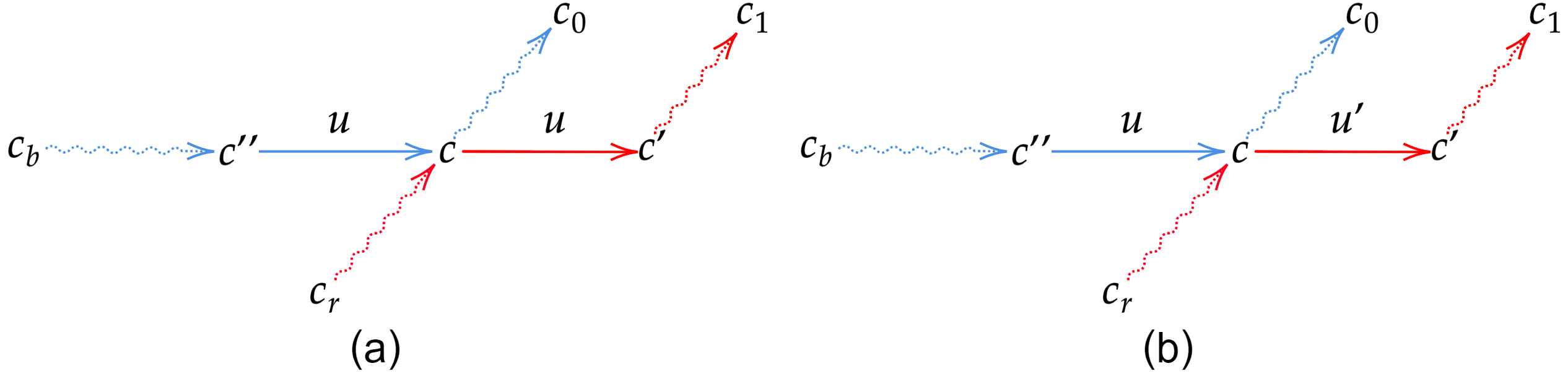}
\caption{Exemplifying the definition of a junction node. In (a), the node $c$ is a junction between the red path that starts at $c_r$ (and ends at $c_0$) and the blue path that starts at $c_b$ (and ends at $c_1$). The reason is that it receives the user $u$ from $c''$ along a blue edge, and passes $u$ along a red edge. In (b), however, $c$ is not a junction, since the user it passes onward along the red path is $u'$, which is not the user $c$ receives along the blue path.}
\label{fig:intersection}
\end{figure}

Next, we employ Algorithm \ref{alg:green_flow} on the blue and red paths in $E^1 \cup E^{0,1}$. We show that
\begin{lemma}\label{lemma:green_flow}
The output $E$ of Algorithm \ref{alg:green_flow} satisfies the following:
\begin{enumerate}
    \item $E$ is consistent.
    \item When applied to $X$, the resulting matching $X^0$ is feasible. \label{lemma:item:feasible}
    \item $\sum_{e\in E}w(e) \geq \sum_{e\in E^{0,1}}w(e) $.
\end{enumerate}
\end{lemma}
The proof of Lemma \ref{lemma:green_flow} appears below. 
As elaborated above, the matching $X$ and the relocation triplets in $E$ uniquely define the matching $X^0$. By the second part of Lemma \ref{lemma:green_flow}, $X^0$ is feasible. Moreover, by the third part of Lemma \ref{lemma:green_flow} and the definition of relocation triplets, we have
\[g(C\cup \{c_0,c_1\})-g(C\cup\{c_1\})= \sum_{e\in E^{0,1}}w(e) \leq \sum_{e\in E}w(e)= F(X^0)-g(X).
\]
This completes the proof of the theorem.
\begin{algorithm}
\caption{Flow Construction for $X^0$ \label{alg:green_flow} }
\begin{algorithmic}[1]
\STATE let $B$ be the set of blue paths and $R$ be the set of red paths.
\STATE let $E\gets \emptyset$ be the set of new paths.
\WHILE {$R \neq \emptyset$}{
\IF{there is a junction node w.r.t. $B,R$ \label{line:if}}
{
\STATE let $c$ be a junction node, and $b,r$ denote the paths whose edges $(c'',c,u)\in b$, $(c,c',u)\in r$ form the junction such that $c$ is the closest junction node to the sink of $b$. \label{line:identified}
\STATE add to $E$ the edge $(c'',c,u)$ and all the directed edges that precede it in $b$, and $(c,c',u)$ and all subsequent directed edges in $r$. \label{line:blue_and_red} \label{algline:addition}
\STATE remove $r$ from $R$, $b$ from $B$.
\STATE \textbf{continue}
}\ELSE{
\STATE add all the edges in $R$ to $E$, set $R\gets \emptyset$. \label{algline:addition_all}
} \ENDIF
}\ENDWHILE
\STATE \textbf{return} $E$
\end{algorithmic}
\end{algorithm}
\subsubsection{Proofs of Proposition \ref{prop:no_cycle_and_sink} and Lemma \ref{lemma:green_flow}}
\begin{proof}[Proof of Proposition \ref{prop:no_cycle_and_sink}]
For (1), assume by contradiction that a simple cycle $e_1,e_2,\dots e_k$ exists for some $k\in \mathbb N$. 
Since $c_1 \notin \mC(X)$, there is no relocation triplet with $c_1$ in the first entry, and hence $c_1$ does not participate in the cycle. We proceed by analyzing the weight of the cycle, $\sum_{i=1}^k w(e_i)$. \footnote{In general, a set of relocation triplets can contain cycles with positive/negative weights, if providers pass different users along the cycle. However, as we prove, this cannot happen in $G^1$ due to the optimality of $X^1$.}
\begin{itemize}
    \item If $\sum_{i=1}^k w(e^i) = 0$, we can remove the cycle from the graph and obtain a new graph $\tilde{ G^1}$ and a corresponding matching $\tilde {X^1}$. Observe that the number of users  every provider gets is the same as in $X$ (see Equation \eqref{eq:num_of_users}), and hence not only 
    $\mC(\tilde {X^1})=\mC(X^1)=C\cup \{c_1\}$ but also every provider in that set meets her threshold. Further, we did not change the sum of weights, and $F(X^1)-F(X)=F(\tilde {X^1})-F(X)$ implies $F(X^1)=F(\tilde {X^1})$; hence, $\tilde {X^1}$ is also optimal and we can assume w.l.o.g. that $X^1$ does not contain such cycles.
    \item If $\sum_{i=1}^k w(e^i) > 0$, we denote by $\tilde X$ a matching such that
    \[
    \tilde X(u)=
    \begin{cases}
    c' & \text{if the edge $(c,c',u)$ belongs to the cycle}\\
    X(u) &\text{otherwise}
    \end{cases}.
    \]
    Since the number of users each provider in $C=\mC(X)=\mC(\tilde X)$ gets under $\tilde X$ is the same as under $X$, $\tilde X$ is feasible. Moreover, $F(\tilde X)-F(X)>0$; hence, we obtain a contradiction to the optimality of $X$.
    \item If $\sum_{i=1}^k w(e^i) < 0$, we can use an argument similar to the previous case to claim sub-optimality of $X^1$.
\end{itemize}
For (2), assume by contradiction that a node $v\in \mC, v\neq c_1$ is a sink, and observe that we must have $v\in C$ since $\mC(X^1) = C \cup \{c_1\}$. Let $v_1,\dots, v_k,v$ denote the shortest path ending at $v$. Because $c_1\notin \mC(X)$, we know that $c_1$ cannot participate in this path. Further, $X$ is feasible and hence $v$ gets at least $\nu_v$ users under $X^1$. The analysis identically to the first part of the proposition, arguing that the contradiction assumption entails the existence of a path with positive/negative weights, in contrast to the optimality of $X$ and $X^1$.
\end{proof}
\begin{proof}[Proof of Lemma \ref{lemma:green_flow}]
Assume by contradiction that the output $E$ is not consistent. By definition of consistency, there exists an edge $e=(c,c',u)$ such that
\begin{enumerate}
    \item $X(u)\neq c$, and
    \item $(c'',c',u)\notin E$ for every $c\in \mathcal C$.
\end{enumerate}
Notice that $e\in E \subseteq  E^1 \cup E^{0,1}$; hence, $e$ is either blue or red. If $e$ is blue, let $b(e)$ denote the path $e$ is part of. Since $e\in E$ and is blue, the only way it could have been added to $E$ is via Line \ref{line:blue_and_red}. This means that either $e$ is the first edge in $b(e)$, in which case $X(u)=c$ since $E^1$ is consistent; or $e$ is an intermediate edge in $b(e)$, in which case there exists another edge $e'=(c'',c',u)\in b(e)$ that precedes it, again because $E^1$ is consistent. In both cases, we obtain contradiction.

Otherwise $e$ is red. Let $r$ denote the path that contains $e$. We have two cases:
\begin{itemize}
    \item If $c$ is a junction w.r.t. the initial $B,R$. In this case, there exists a blue path $b\in B$ that contains an edge $(c'',c,u)$, by the definition of a junction node. Moreover, at some point in the execution $e$ was added, so $b$ must have been identified as a path containing an edge that forms a junction node $v$ (not necessarily $c$) in Line \ref{line:if}. Recall that in Line \ref{line:identified} we assume that $v$ is the closest junction node to the sink of $b$, which is $c_1$; hence, all edges of $b$ that precedes the outgoing edge from $v$ are added to $E$ too, including $(c'',c,u)$. This implies a contradiction.
    \item Else, $c$ is not a junction. If $e$ is the first edge in the red path $r$, then $X^1(u)=c$, and since $c$ is not a junction, $X(u)=c$ as well. This holds because both $X,X^1$ are feasible. Otherwise, if $e$ is an intermediate edge in $r$, then there must exists a red edge $(c'',c,u)$ for some $c'' \in C$, because $E^1 \cup E^{0,1}$ is consistent. Since red edges like $e$ are inserted to $E$ in Lines \ref{line:blue_and_red} and \ref{algline:addition_all}, the preceding edges in their red path, including $(c'',c,u)$, are added as well. In both cases, we reach a contradiction.
\end{itemize}
\paragraph{Second part} Denote the matching obtained by applying the relocation triplets of $E$ to $X$ by $X^0$. To show that $X^0$ is feasible, we need to show that for every $c\in C\cup \{c_0 \}$, it holds that  $\abs{\{u:X^0(u)=c\}}\geq \nu_c$. To do so, we rely on the feasibility of $X^1$ and $X^{1,0}$, whose relocation edges were used to construct $X^0$. We divide the analysis into three parts:
\begin{itemize}
    \item If $c=c_0$. Since $X^{1,0}$ is feasible, we know that the $\deg ^{+}(c_0)$ in $G^{0,1}$ is at least $\nu_{c_0}$ (Recall the quantification of the number of matched users in Equation (\ref{eq:num_of_users})). Since $E$ contains the final edge of every red path, the indegree of $c_0$ in $G^0=(\mathcal C,E,w)$ is the same as in $G^{0,1}$.
    \item Else, if $c$ is the source of at least one path in $E$. In this case, it must have been the source of some paths in $E^1$ (blue) and $E^0,1$ (red). Recall that the sink of every blue path is $c_1$, and the sink of every red path is $c_0$. Moreover, if $c$ participates in other red/blue paths, it must be an intermediate node; thus, we can analyze its loss of users due to the paths in which $c$ is the source solely. Since $E\subseteq E^1 \cup E^{0,1}$, $c$ is the source of less paths in $G^0=(\mathcal C,E,w)$ than in $G^{0,1}$; therefore, its indegree in $G^0$ is  greater or equal to its indegree in $G^{0,1}$, which implies that $X^0$ matched $c$ with at least as many users as $X^{0,1}$.
    \item Finally, for any other $c$, $X^0$ matches $c$ with the same number of users as $X^{0,1}$, since its difference between the indegree and the outdegree in $G^0=(\mathcal C,E,w)$ remains as in $G^{0,1}$.
\end{itemize}

\paragraph{Third part}
The proof of this part is based on the following observation:
\begin{observation}\label{obs:p}
Let $b$ be a blue path with source $c_b$ and sink $c_1$, $r$ be a red path with source $c_r$ and sink $c_0$, and let $c$ be a junction w.r.t. $b$ and $r$, with edges $(c'',c,u)\in b$ and $(c,c',u)\in r$. Denote by $p$ the path that starts from $c_b$, takes the edge $(c'',c,u)$ and the edges that precedes in $b$, and then takes $(c,c',u)$ and its subsequent edges in $r$, ending at $c_0$. Then,  $\sum_{e\in p}w(e) \leq \sum_{e\in E^{0,1}}w(e) $. 
\end{observation}
To see why Observation \ref{obs:p} holds, recall that the prefix of $p$ from its source to $(c'',c,u)$ inclusive, all blue edges, must have higher weight than the prefix of $r$ from its source to $(c,c',u)$, exclusive. This is true since otherwise we could find a heavier blue path to replace $b$ in $E^0$. However, this cannot be true as $X^1$, which accounts for the blue edges in $G^1$, is an optimal matching for $g(C\cup \{ c_1\})$. Finally, Algorithm \ref{alg:green_flow} adds red paths either in their entirety (Line \ref{algline:addition_all}) or by modifying them to be heavier according to Observation \ref{obs:p}; hence, $\sum_{e\in E}w(e) \geq \sum_{e\in E^{0,1}}w(e)$.
\end{proof}
\subsubsection{From Deterministic to Stochastic Matching}
To complete the picture, it remains to argue that problems in which user queries contribute non-unit amounts to provider viability can be reduced to the unit case analyzed in the previous sections. We can think of the optimization problem in ~\eqref{eq:ilp_linear_util} as a matching problem with a weighted constraint 
\begin{align}
\label{eq:general_problem_weighted}
X^\ast =  \arg\max_{X,Y} & \sum_{u \in {\cal U}} \left(\sum_{t=1}^{D(u)}\sum_{c\in{\cal C}} A_{uc}X_{uct}  \right)\nonumber\\ 
\text{subject to }& \sum_{c \in {\cal C}} X_{uct} = 1\quad \forall u\in{\cal U}, t\in \{1,\ldots, D(u)\}\nonumber\\
& X_{uc} \leq Y_c\quad \forall u\in{\cal U}, c\in{\cal C}\nonumber\\
& \sum_{u\in{\cal U}}\sum_{t=1}^{D(u)}w_{uct} X_{uct} \geq \nu_cY_c, \quad\forall c\in {\cal C}\nonumber\\
& X_{uct} \in [0,1], Y_c \in \{0,1\}, \quad \forall u\in{\cal U}, c\in{\cal C},t\in [1,\ldots, D(u)] ,
\end{align}
where the weight $w_{uct}$ reflects the expected engagement of user $u$ towards provider $c$ at time $t$. Similarly to the deterministic setting above, we aim to show that 
\begin{align}
\label{eq:local_problem_weighted}
    g(C) \mapsto \max_X &  \sum_{u \in {\cal U}} \left(\sum_{t=1}^{D(u)}\sum_{c\in C} A_{uc}X_{uct}  \right)\nonumber\\
    \text{subject to }& \sum_{c\in C} X_{uct} = 1\quad \forall u\in{\cal U}, t\in [1,\ldots, D(u)]\nonumber\\
   & \sum_{u\in{\cal U}}\sum_{t=1}^{D(u)} w_{uct}X_{uct} \geq \nu_c \quad\forall c\in C \nonumber\\
   & X_{uct} \in [0,1], \quad \forall u\in{\cal U}, c\in{\cal C},t\in [1,\ldots, D(u)]
\end{align}
is submodular. 
The challenge in this case is that the problem of ~ \eqref{eq:local_problem_weighted} is no longer totally unimodular due to the fractional coefficients introduced by the viability constraint, hence, the combinatorial argument of the previous section is no longer applicable. 
It is possible, however, to construct an unweighted equivalent to weighted problem by introducing fictitious users and providers in a symmetric fashion. Applying the submodularity argument to the unweighted problem implies that the weighted one is submodular as well. See the extended version of this paper~\cite{fullversion} for a complete proof of this fact.

\subsection{Non-linear Optimization via Column Generation}\label{sec:column-gen}
\subsubsection{Formulation}
As discussed in Sec.~\ref{sec:columnGeneration}, it is desirable to have a procedure that can optimize social welfare under non-linear utility models. To this end, we extend the mixed-integer linear program in Problem~\eqref{eq:ilp_linear_util} to handle non-linear utilities. Let $C\in \calC^k$ be a $k$-tuple of providers. A pair $(q_u, C)\in \calQ \times \calC^k$ represents a possible answer to user $u$'s $k$ queries identical to $q_u$ by the provider tuple $C$. We call such a tuple a 
\emph{star} $q_uC$. For each star, we use a variable $\pi_{q_uC}$ to represent the policy's match to $q_u$.
\label{sec:non_linear_opt}
{
\begin{maxi}
{\pi, y}{ \sum_{u\in \calU}\sum_{q_u \in \cal Q}\sum_{C \in \calC^k}\pi_{q_u,C}\bar{\sigma}(q_{u}, C)}
{
\label{eq:ilp_nonlinear_util_supp}
}{}
\addConstraint{\sum_{C \in \calC^k}\pi_{q_u, C}}{\leq1\quad}{u\in\calU}
\addConstraint{\sum_{\{C\in \calC^k| c\in C\}}\pi_{q_u, C}}{\leq y_c\quad}{ u\in\calU, c\in\calC}
\addConstraint{\sum_{u\in\calU}\sum_{C\in \calC^k} \#[q_uC,c]\overline{Q}(q_u)\pi_{q_u, C}}{\geq \nu_c y_c,\quad}{c\in\calC,}
\end{maxi}
}
where $\#[uC,c]$ is the number of times provider $c$ appears in star $q_uC$, and $\bar{\sigma}(q_{u}, C)=\rho(u)P_u(q_u)\sigma(q_{u}, C)$. 
We rely on the linear relaxation of the
integrality constraints to approximate the
solution of~\eqref{eq:ilp_nonlinear_util_supp} efficiently. It is not obvious if and how this problem can be approximated via discrete algorithmic techniques, 
so we resort to relaxing the integrality constraints and solving the problem as a linear program. 
Even under the linear relaxation, the problem size still grows proportionally to $\calC^k$ due to the number of variables introduced by linearization. The redeeming property of this problem, however, is that the number of constraints grows proportionally to $\calU \times \calQ \times \calC$ and not $\calC^k$. Hence, it is feasible to approach the problem from a column generation perspective. 

\subsubsection{Column Generation}
A standard column generation approach for solving a large linear program is a two-step iterative algorithm in which the LP is initially constructed using a small subset of its variables to obtain a reduced-size (master) problem. The dual of the master problem yields a dual optimal solution, which is then used to find a (as of yet not generated) variable with maximal reduced cost. That variable is added to the master problem. The method iterates until no variable with positive reduced cost can be found, or some convergence tolerance is reached.

When the set of primal variables is large, the problem of finding a variable with maximal reduced cost (also called a column generation oracle) is still a hard combinatorial optimization problem (typically some flavor of knapsack). However, these problems tend to be massively decomposable and the running time does not scale exponentially in practice. 

We now proceed to derive a column generation oracle for Problem~\eqref{eq:ilp_nonlinear_util_supp}. Let ${\cal A} = (A, b, c)$ denote an LP in inequality form, denoting the optimization problem $x^\ast = \arg\max_{x: Ax\leq b, x\geq 0} c^Tx$. Let $y^*$ be an optimal dual solution to ${\cal A}$. The reduced cost problem is thus $\hat{c} = c - A^Ty^\ast$. The column generation oracle thus solves the problem $i^\ast = \arg\max_i \hat{c}$, which corresponds to the index of the primal variable with highest reduced cost.  We now discuss solving the column generation problem given the specific form of \eqref{eq:ilp_nonlinear_util_supp}.

We adopt the following convention for naming the dual variables corresponding to constraints in  \eqref{eq:ilp_nonlinear_util_supp}:
\begin{align*}
    \beta_{u} : & \sum_{C \in \calC^k}\pi_{q_u, C}\leq 1\quad & u\in\calU,\\
    \gamma_{uc}:  & \sum_{\{C\in \calC^k| c\in C\}}\pi_{q_u, C}\leq y_c\quad & u\in\calU, c\in\calC,\\
   \alpha_{c} :  & \sum_{u\in\calU}\sum_{C\in \calC^k} \#[q_uC,c]\overline{Q}(q_u)\pi_{q_u, C} \geq \nu_c y_c\quad & c\in\calC.\\
\end{align*}

The column generation problem (derived by computing the dual and maximizing the reduced cost) then becomes: 

\[uC^\ast = \arg\max_{u\in\calU, C\in\calC^k}\quad \bar{\sigma}(q_{u}, C) - \left(\beta_u + \sum_{c \in C} \gamma_{uc} - \sum_{c\in C}\#[q_uC,c]\overline{Q}(q_u)\alpha_{c}\right).\]

Let us now discuss how the above maximization can be solved. First, observe that the problem decomposes in the user variable $u$. That is for each $u \in \calU$, we can independently solve the maximization over $C$. This can be done in parallel for each user and the maximum over $u$ can be computed by enumeration. 
Supposing $u$ is fixed, we still have to solve a series of non-linear integer optimization problems due to the non-linear nature of $\bar{\sigma}$. We can covert the non-linear problems to linear in two steps. First, we convert the tuple maximization problem to a binary-variable one as by introducing slot indicator variables for each of the elements of the tuple $C$. That is:
\[\max_{x}\quad \bar{\sigma}\left(\sum_{t\in 1:k}\sum_{c}x_{ct} A_{uc}\right) - \left(\beta_u + \sum_{t\in 1:k}\sum_{c \in \calC} x_{ct}\gamma_{uc} - \sum_{c\in \calC}\left(\sum_t x_{ct}\right) \alpha_{c}\right) \text{ s.t. } \sum_c x_{ct}  = 1\quad\forall t\in 1:k\ .\]

Furthermore, the non-linear $\bar{\sigma}$ can be replaced by a series of local first-order approximations (in fact, a zero-order approximation is also possible), to yield binary integer program. That is:
\[\max_{x}\quad \bar{\sigma}^\prime(m_i) \cdot \left(\sum_{t\in 1:k}\sum_{c}x_{ct} A_{uc}\right) - \left(\beta_u + \sum_{t\in 1:k}\sum_{c \in \calC} x_{ct}\gamma_{uc} - \sum_{c\in \calC}\left(\sum_t x_{ct}\right) \alpha_{c}\right)\]
\[ \text{ s.t. } \sum_c x_{ct}  = 1\quad\forall t\in 1:k\ ,\quad l_i\leq \sum_{t\in 1:k}\sum_{c}x_{ct} A_{uc}\leq u_i,\]
where $\bar{\sigma}^\prime(m_i)$ is the derivative of $\bar{\sigma}$ at $m_i$. Under smoothness assumptions on $\bar{\sigma}$, this linearization provides a bounded approximation to the original problem. Again, these interval sub-problems can be solved in parallel.

\subsubsection{Illustrative Evaluation of Column Generation} 
\label{sec:column-gen-expers}

Here we describe some preliminary experiments using the column generation strategy described above.
We find that column generation is capable of keeping more providers viable than the myopic baseline at early steps.
However the performance was less reliable than the LP-RS approach that was evaluated in Section \ref{sec:experiments}.
In some settings, the column generation approach fails to maintain a consistent matching in successive iterations, resulting in a slowly declining number of viable providers over time.
We hypothesize that this is due to rounding errors in the procedure, or early stopping before convergence (our implementation used $300$ iterations of column generation rather than running exhaustively until convergence).
Therefore we expect that improvements can be made by fine-tuning this approach, but leave this to future work.

Figure \ref{fig:column-gen-results} shows the results of the experiments.
While we use the same embeddings data as in the main body of the paper, but we scale down the problem size to compensate for the slower runtime of the column generation approach; this explains the differing number of viable providers at equilibrium compared with the $LP-RS$ approach presented in Section \ref{sec:experiments}.
In the synthetic setting we used $50$ providers, about $260$ users and viability threshold of $\nu=5$.
In the other two datasets we used a competitive (from the provider perspective) setting of $100$ providers, $100$ users, and viability threshold of $\nu=8$.
We used slate size of $1$ for all datasets.

\begin{figure*}[t!]
  \centering
  \begin{subfigure}[t]{\subfigwidth}
    \noindent\resizebox{\textwidth}{!}{
      \includegraphics[width=\subfigwidth]{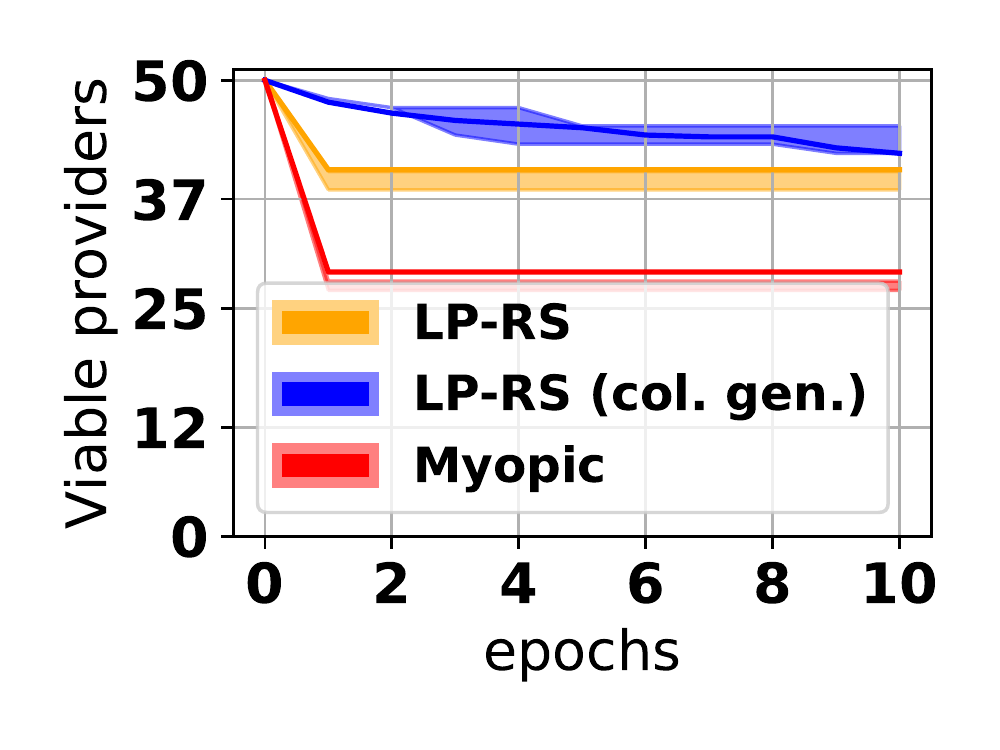}
    }
  \caption{
  Synthetic embeddings.
  }\label{fig:synthetic_viability_appendix}
  \end{subfigure}%
  \begin{subfigure}[t]{\subfigwidth}
    \noindent\resizebox{\textwidth}{!}{
      \includegraphics[width=\subfigwidth]{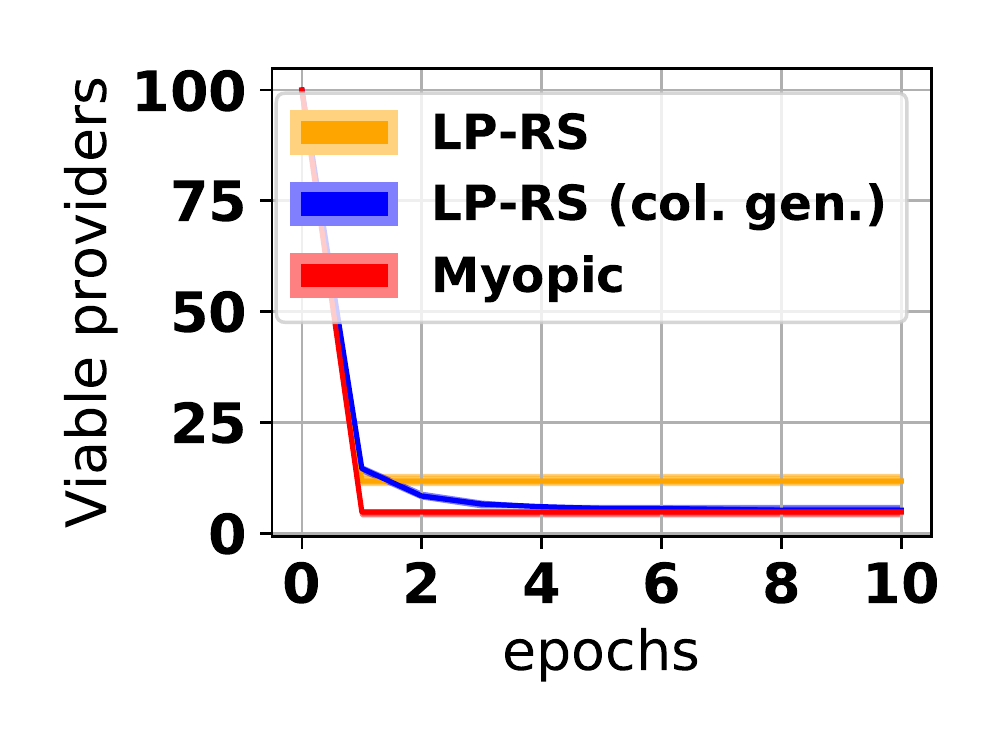}
    }
  \caption{
  Movielens embeddings.
  }\label{fig:movielens_viability_appendix}
  \end{subfigure}%
  \begin{subfigure}[t]{\subfigwidth}
    \noindent\resizebox{\textwidth}{!}{
      \includegraphics[width=\subfigwidth]{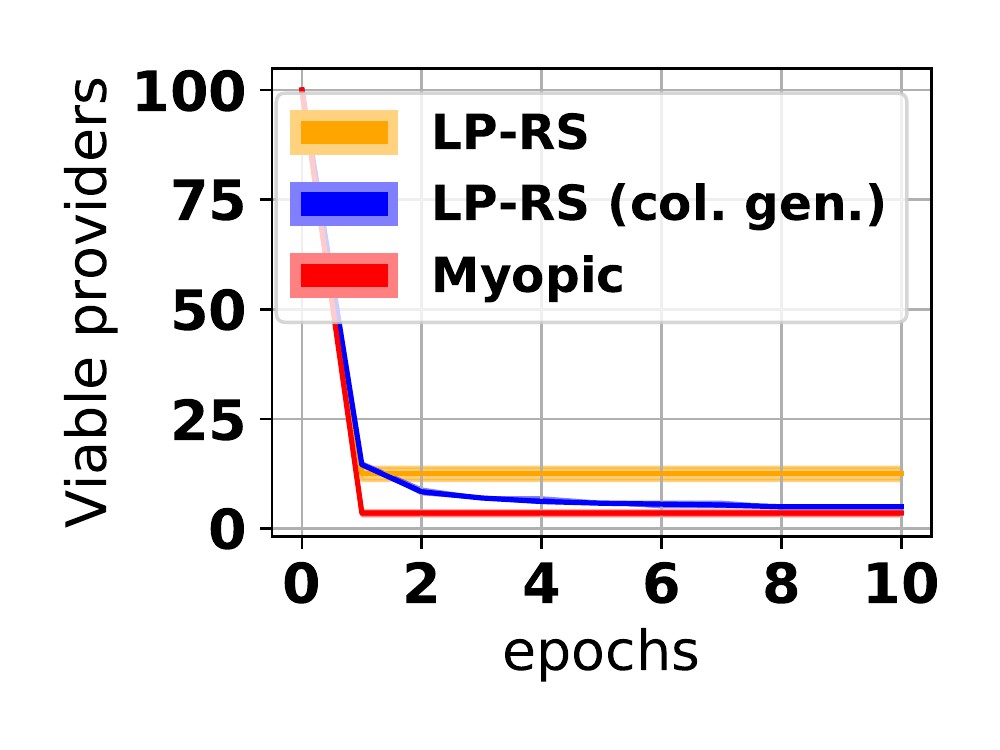}
    }
  \caption{
  SNAP embeddings.
  }\label{fig:snap_viability_appendix}
  \end{subfigure}%
  \caption{
    Simulations that evaluate the column generation matching strategy (LP-RS col. gen.) on smaller problems.
    While LP-RS col. gen. is capable of finding a good matching at a given step of simulation, its solutions are relatively inconsistent compared with LP-RS in the limit of several time steps.
  }\label{fig:column-gen-results}
\end{figure*}

\section{Training Details}
\label{sec:embeddings-details}

Here we provide details of training for the embeddings described in Section \ref{sec:datasets}.

\paragraph{Movielens}
We trained a non-negative matrix factorization embedding space using the Movielens dataset \citep{harper2015movielens}.
We use the distribution of this dataset containing about $100,000$ ratings of about $9,000$ movies by about $600$ users.
The dataset comprises a sparse ratings matrix Given the sparse ratings matrix $R \in \R_{\geq 0}^{N_\text{users} \times N_\text{movies}}$.
We use the binarized engagement matrix $E \in \{0, 1\}^{N_\text{users} \times N_\text{providers}}$ with $E_{i,j} = \mathbbm{1}(R_{i,j})$.
The embeddings are produced by finding low-rank non-negative factors of the engagement matrix $E \in \{0, 1\}^{N_\text{users} \times N_\text{providers}}$, by solving the optimization problem
\begin{equation}\label{eq:nmf}
  \min_{U,V} || (E - U V^T) ||_F^2
    + \lambda_U ||U||_F^2 + \lambda_V ||V||_F^2
\end{equation}
which yields factors $U \in \R_{\geq 0}^{N_\text{users} \times N_\text{topic}}$ and $U \in \R_{\geq 0}^{N_\text{providers} \times N_\text{topic}}$.

The factors $U \in \R_{\geq 0}^{N_\text{users} \times N_\text{topic}}$ and $C \in \R_{\geq 0}^{N_\text{providers} \times N_\text{topic}}$ yield row and column vectors that are treated as the embedding vectors; in this case, a single content provider is equivalent to a single movie from the dataset.

The rows of these factor matrices were used to sample user and provider vectors in the RS ecosystem.
Note that the value of the ratings were not used, so the ``affinity'' between user and movie in this embedding space is a measure of how likely the user is to \emph{watch} the movie, rather than rate it highly.
The randomly initialized factors $U, V$ are alternatively updated via Weighted alternating least squares \cite{hu2008collaborative} for 100 iterations.
We used embedding rank $N_\text{topic}=20$, and set $\lambda_U = 1$ and $\lambda_v = 1.$. 

\paragraph{SNAP}
The dataset consists of a large list of (followee, follower) pairs, where each user is
given a unique node ID label. We turn this dataset into a set of providers and users as follows.
First, we randomly subsample 100k of the 41 million users. We designate followees as
providers. For every provider, we then remove their follow edges, so that they do not follow anyone else.
This makes the graph bipartite, where users follow providers.
We then choose the top 500 providers in terms of follower count, and remove any users that do not follow
at least one of them. This leaves a total of 500 providers, and 59,394 users.

for each user $i$, we learn a 24-dimensional vector $\mathbf{u}_i$, and for each provider $j$, a 24-dimensional
vector $\mathbf{v}_j$. We train these embeddings by cross-entropy to predict whether there is an edge $A_{ij}$
between user $i$ and provider $j$, where the probability is given by,
\begin{align}
    P(A_{ij} = 1) &= \sigma(\mathbf{u}_i^\top \mathbf{v}_j)
\end{align}
Where $\sigma(\cdot)$ is the sigmoid function. $A_{ij}$ is 1 if user $i$ follows provider $j$.
We add a small amount of weight decay to ensure that the embeddings are well behaved.

\section{Simulation Details}
\label{sec:simulation-details}

This section contains details to reproduce the simulations described in Section \ref{sec:results}.

\paragraph{Exploring Embedding Type}
We generate 50 provider vectors and a varying number of user vectors (between 4412 and 4672 per run).
Provider and user vectors are sampled in a 10-dimensional topic space, with provider vectors sampled normally with variance 50.
These provider vectors serve as cluster means for the mixture-of-Gaussians that generates user vectors. 
The prior over cluster assignments depends on the variant (\emph{uniform} vs. \emph{skewed} described in the text).
User variance was set to 0.1 in the \emph{uniform} variant, and user variance scaled inversely with popularity in the \emph{skewed} variant.
The slate size was $s=1$, with viability threshold set to $\nu=80$.
We run for 10 epochs using 5 seeds for each method/data type pair, and report average values plus or minus one standard deviation in Table \ref{tab:embedding-type}.

\paragraph{Tradeoffs in regret and welfare}
We generate synthetic embeddings of the \emph{skewed} variant, with 50 providers and around 900 users.
We use slate size $s=4$ with viability threshold $\nu=9$.
Other settings are carried over from the previous experiment.
We run steps of simulation until the policies converge then measure the welfare and max regret metrics.

\paragraph{Synthetic simulation}
We generate synthetic embeddings of the \emph{skewed}.
The parameters are similar to those described above, with 50 providers (distributed normally with $\sigma^2 = 5$) and about $10,000$ users.
We simulate the RS for ten epochs with slate size $s=1$ and viability threshold $\nu=78.5$.

\paragraph{Movielens simulation}
Starting with the learned low-rank factors, we subsample $250$ movie column (which serve as providers) and $1,000$ user columns.
We simulate the RS for ten epochs with slate size $s=1$ and viability threshold $\nu=10$.

\paragraph{SNAP simulation}
Starting with the learned embedding, we subsample $300$ providers and $566$ users.
We simulate the RS for ten epochs with slate size $s=1$ and viability threshold $\nu=10$.

\section{Stochastic Policy Ablation}\label{sec:stochastic_policy_ablation}

When considering which providers to recommend to a particular user query $q_u$, our proposed policy $\pi_{\text{LP-RS}}$ may choose to ``subsidize'' providers that are \emph{slightly suboptimal}, i.e. not the best affinity for the user but still having relatively good affinity.
By contrast the myopic policy ignores the ecosystem dynamics and always chooses the best-affinity provider for each user.
Would a policy that stochastically samples providers $c$ in proportion to their affinity to $q_u$ naturally lead to a similar subsidizing effect as $\pi_{\text{LP-RS}}$?
We find empirically that this is not the case.

\begin{figure*}[h!]
  \centering
      \includegraphics[width=.6\textwidth]{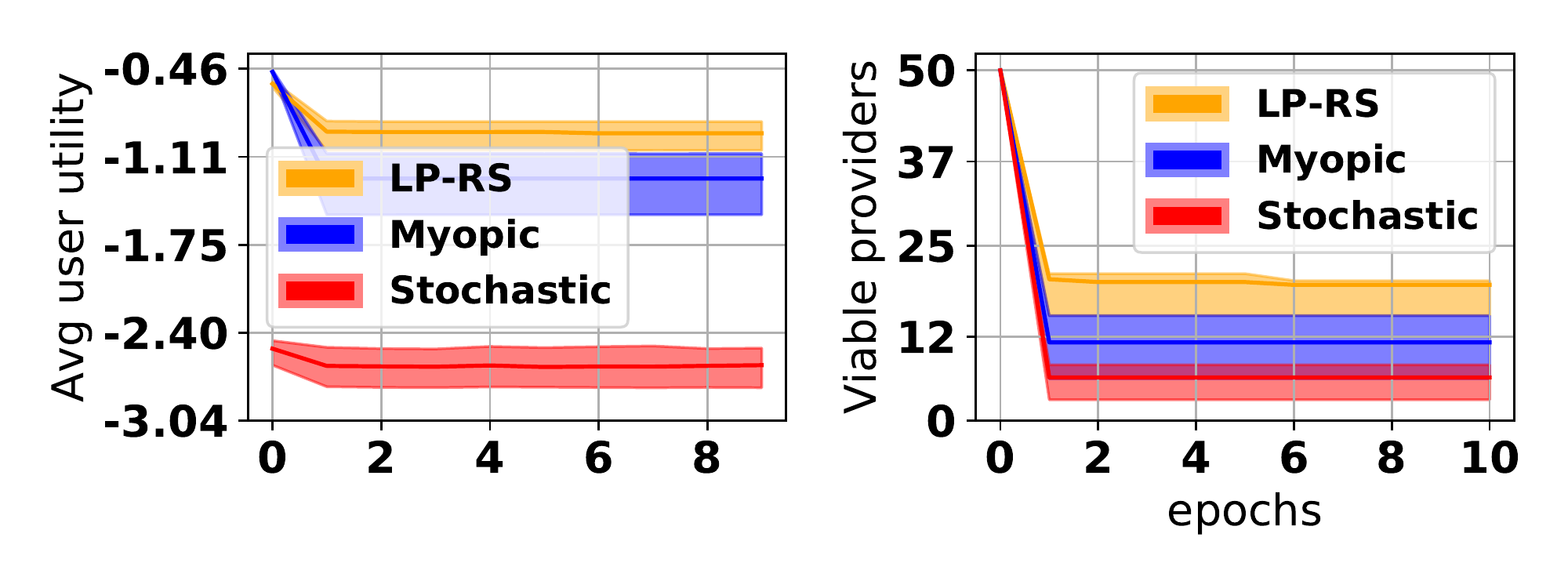}
  \caption{
  An affinity-aware stochastic policy under-performs relative to the myopic baseline and the proposed LP-RS policy.
  }\label{fig:stochastic_policy_ablation}
\end{figure*}

We specify a stochastic policy $\pi_{\text{Stochastic}}$ that samples the recommended provider for user query $q_u$ as $c^{\text{Sto}}_{q_u} \sim p(c;q_u)$ where $p(c;q_u) = \frac{r(q_u, c)}{\sum_{c'}r(q_u, c')}$ is a Boltmann distribution over creators specified by the affinity function $r(q_u, \cdot)$ for that user.
If multiple recommendations per user query are needed, then the appropriate number of samples are drawn without replacement.
We simulate the recommender ecosystem using synthetic embedding distributions with the same settings as in Section \ref{sec:experiments}.
Figure \ref{fig:stochastic_policy_ablation} shows that the stochastic baseline under-performs relative to the myopic policy, indicating that simply sampling suboptimal providers with some non-zero probability does not keep these providers viable in the long run.

\section{Additional Histograms}
\label{sec:more-histograms}

Figure \ref{fig:user_utility_histogram} shows user utility histograms for all simulations.

\begin{figure*}[t!]
  \centering
  \begin{subfigure}[t]{\subfigwidth}
    \noindent\resizebox{\textwidth}{!}{
      \includegraphics[width=\subfigwidth]{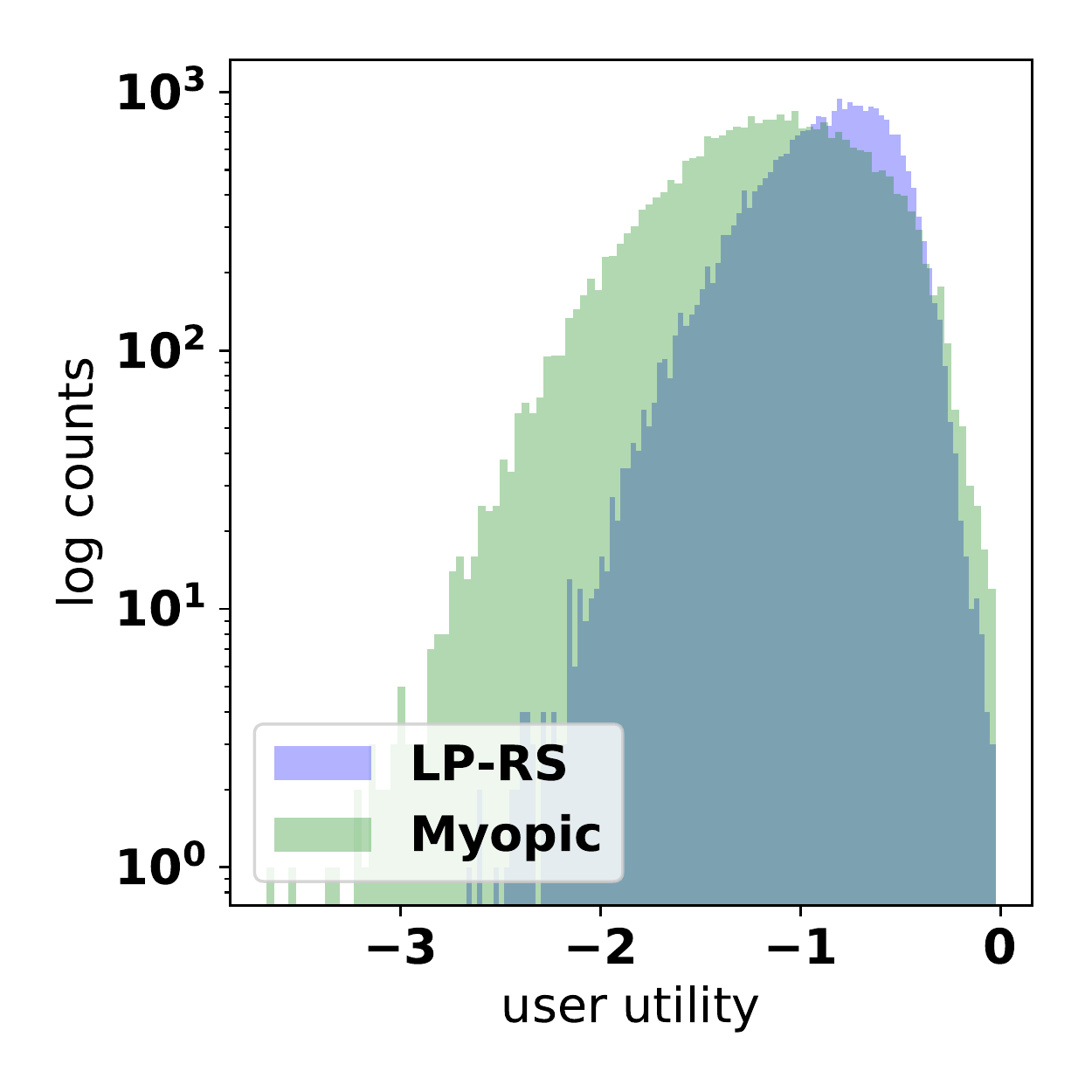}
    }
  \caption{
  Synthetic embeddings.
  }\label{fig:synthetic_user_utility_histogram}
  \end{subfigure}%
  \begin{subfigure}[t]{\subfigwidth}
    \noindent\resizebox{\textwidth}{!}{
      \includegraphics[width=\subfigwidth]{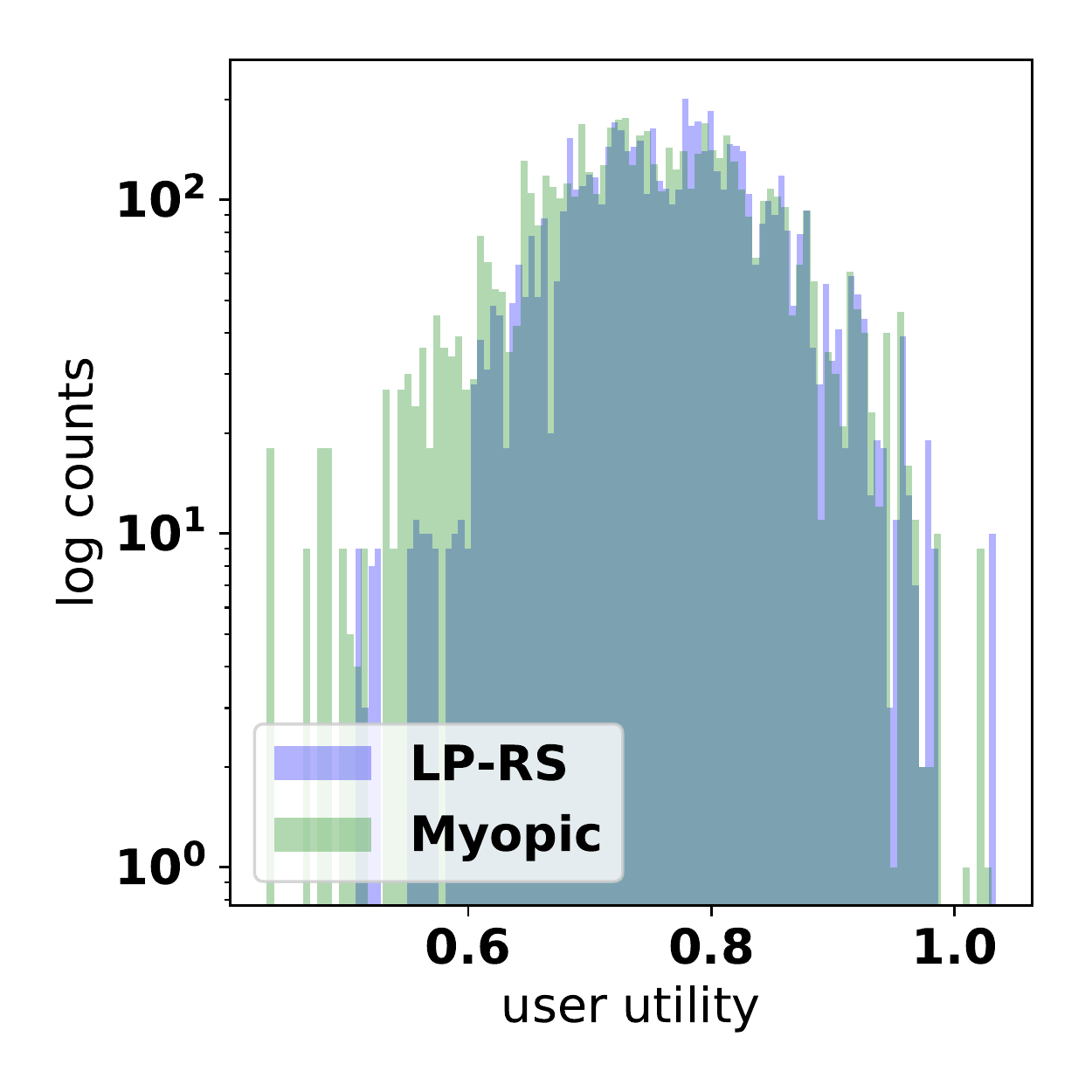}
    }
  \caption{
  Movielens embeddings.
  }\label{fig:movielens_user_utility_histogram}
  \end{subfigure}%
  \begin{subfigure}[t]{\subfigwidth}
    \noindent\resizebox{\textwidth}{!}{
      \includegraphics[width=\subfigwidth]{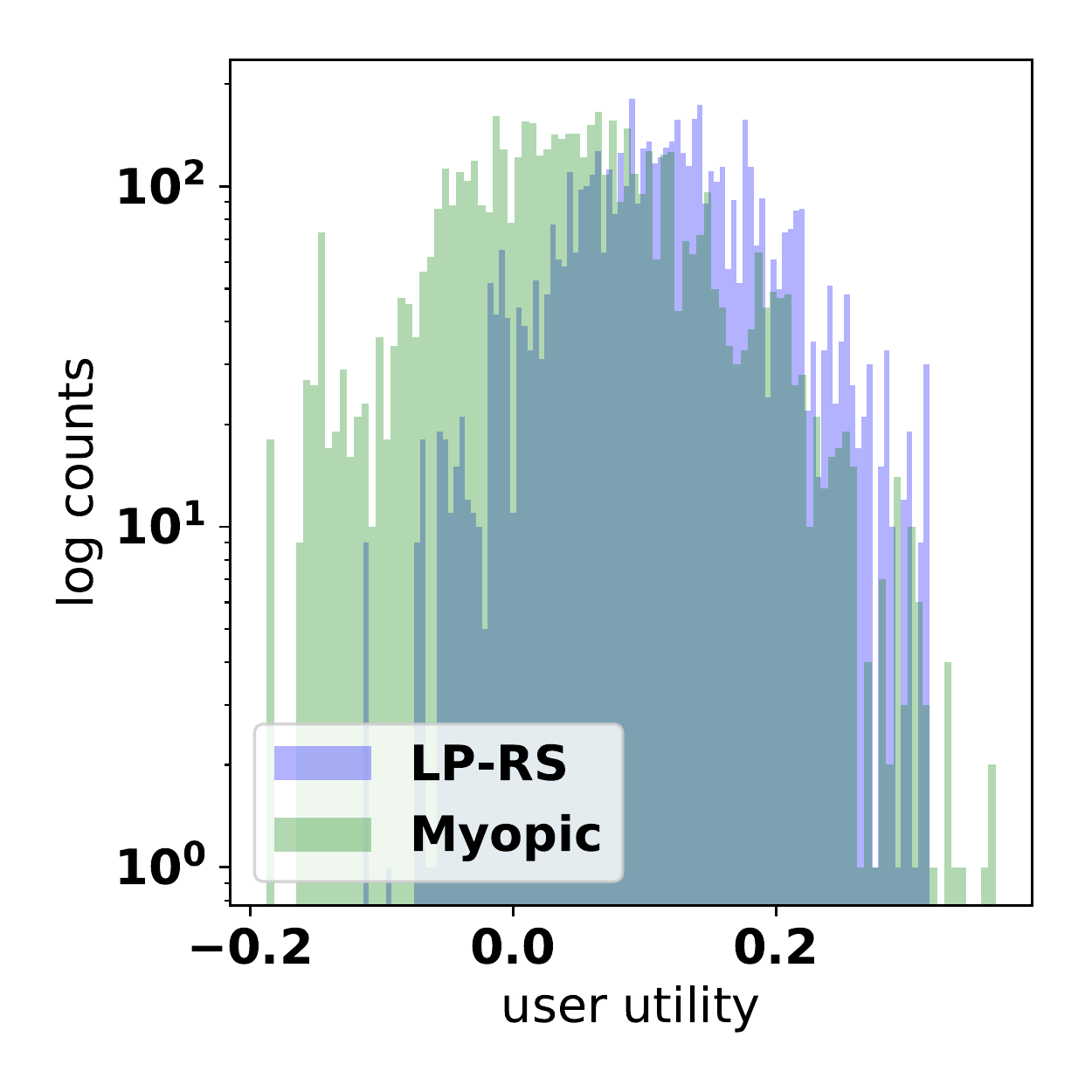}
    }
  \caption{
  SNAP embeddings.
  }\label{fig:snap_user_utility_histogram}
  \end{subfigure}%
  \caption{
    Histogram of user utility.
  }\label{fig:user_utility_histogram}
\end{figure*}

\section{Extended Welfare-Regret Tradeoff Results}
\label{sec:extra-table}

\begin{table*}[ht!]
\centering

\begin{tabular}{lllll}
\toprule
    &        &      Avg. Welfare &        Max Regret &  Viable Providers \\
\midrule
0.1 & LP-RS &  18.02 $\pm$ 1.05 &   7.24 $\pm$ 0.77 &  47.20 $\pm$ 1.72 \\
    & Myopic &  13.49 $\pm$ 1.26 &  10.17 $\pm$ 0.20 &  11.80 $\pm$ 1.47 \\
0.18 & LP-RS &  19.79 $\pm$ 1.16 &   7.97 $\pm$ 0.84 &  47.20 $\pm$ 1.72 \\
    & Myopic &  14.83 $\pm$ 1.39 &  11.18 $\pm$ 0.22 &  11.80 $\pm$ 1.47 \\
0.26 & LP-RS &  21.87 $\pm$ 1.28 &   8.84 $\pm$ 0.91 &  46.60 $\pm$ 2.15 \\
    & Myopic &  16.42 $\pm$ 1.54 &  12.37 $\pm$ 0.25 &  11.80 $\pm$ 1.47 \\
0.35 & LP-RS &  24.30 $\pm$ 1.43 &  10.14 $\pm$ 1.18 &  45.20 $\pm$ 3.06 \\
    & Myopic &  18.29 $\pm$ 1.71 &  13.79 $\pm$ 0.27 &  11.80 $\pm$ 1.47 \\
0.43 & LP-RS &  27.17 $\pm$ 1.58 &  11.65 $\pm$ 0.91 &  44.00 $\pm$ 3.16 \\
    & Myopic &  20.50 $\pm$ 1.92 &  15.45 $\pm$ 0.31 &  11.80 $\pm$ 1.47 \\
0.51 & LP-RS &  30.44 $\pm$ 1.79 &  14.19 $\pm$ 1.46 &  44.80 $\pm$ 0.98 \\
    & Myopic &  23.08 $\pm$ 2.16 &  17.40 $\pm$ 0.35 &  11.80 $\pm$ 1.47 \\
0.59 & LP-RS &  34.30 $\pm$ 1.95 &  16.50 $\pm$ 1.73 &  43.00 $\pm$ 1.90 \\
    & Myopic &  26.07 $\pm$ 2.44 &  19.65 $\pm$ 0.39 &  11.80 $\pm$ 1.47 \\
0.67 & LP-RS &  38.67 $\pm$ 2.29 &  20.53 $\pm$ 1.46 &  42.60 $\pm$ 1.50 \\
    & Myopic &  29.51 $\pm$ 2.76 &  22.24 $\pm$ 0.44 &  11.80 $\pm$ 1.47 \\
0.75 & LP-RS &  43.69 $\pm$ 2.61 &  24.61 $\pm$ 1.31 &  41.80 $\pm$ 1.72 \\
    & Myopic &  33.44 $\pm$ 3.13 &  25.21 $\pm$ 0.50 &  11.80 $\pm$ 1.47 \\
0.84 & LP-RS &  49.51 $\pm$ 2.90 &  28.95 $\pm$ 1.11 &  39.80 $\pm$ 3.06 \\
    & Myopic &  37.91 $\pm$ 3.55 &  28.57 $\pm$ 0.57 &  11.80 $\pm$ 1.47 \\
0.92 & LP-RS &  56.15 $\pm$ 3.26 &  33.05 $\pm$ 1.16 &  36.80 $\pm$ 3.97 \\
    & Myopic &  42.94 $\pm$ 4.02 &  32.36 $\pm$ 0.64 &  11.80 $\pm$ 1.47 \\
1.0 & LP-RS &  63.62 $\pm$ 3.58 &  37.89 $\pm$ 1.39 &  34.20 $\pm$ 5.08 \\
    & Myopic &  48.59 $\pm$ 4.55 &  36.62 $\pm$ 0.73 &  11.80 $\pm$ 1.47 \\
\bottomrule
\end{tabular}

\caption{
\label{tab:extra-table}
The discounting factor $\gamma$ allows LP-RS to trade off between average user welfare and max user regret.
}
\end{table*}

Table \ref{tab:extra-table} extends the result from Table \ref{tab:tradeoffs} by including the Myopic baseline recommender.

\end{document}